%% file: revision.tex
\documentclass[lettersize,journal]{IEEEtran}
\usepackage{amsmath,amsfonts,amssymb}
\usepackage{array}
\usepackage{floatrow}
\usepackage[caption=false,font=normalsize,labelfont=sf,textfont=sf]{subfig}
\usepackage{textcomp}
\usepackage{stfloats}
\usepackage{url}
\usepackage{verbatim}
\usepackage{graphicx}
\usepackage{cite}
\usepackage{times}
\usepackage{epsfig}
\usepackage{pifont}
\usepackage{caption}
\usepackage{subfig}

\usepackage{subcaption}

\usepackage{booktabs}
\usepackage{multirow}
\usepackage{algorithm}
\usepackage{algpseudocode}
\usepackage{svg}
\usepackage{comment}
\usepackage{xspace}
\usepackage[pagebackref=true,breaklinks=true,letterpaper=true,colorlinks,bookmarks=false]{hyperref}
\usepackage{tabularray}

\newcommand{\cmark}{\ding{51}}%
\newcommand{\xmark}{\ding{55}}%
\newcommand{\AG}{AGPI$^2$\xspace}%

\newcommand{\MI}{\text{MI}}%
\begin{document}

\title{Adaptive Generation of Privileged Intermediate Information for Visible-Infrared Person Re-Identification}

\author{IEEE Publication Technology,~\IEEEmembership{Staff,~IEEE,}
\thanks{This paper was produced by the IEEE Publication Technology Group. They are in Piscataway, NJ.}
\thanks{Manuscript received April 19, 2021; revised August 16, 2021.}}

\author{
Mahdi Alehdaghi, Arthur Josi,~\IEEEmembership{Student Member,~IEEE}, Rafael M. O. Cruz, Pourya Shamsolameli,~\IEEEmembership{Senior Member,~IEEE} and Eric Granger,~\IEEEmembership{Member,~IEEE}
\thanks{M.~Alehdaghi, A.~Josi, R. Cruz, and E. Granger are with the LIVIA, ILLS, Dept. of Systems Engineering, ETS Montreal, Canada (Emails: \{mahdi.alehdaghi.1, arthur.josi.1\}@ens.etsmtl.ca, \{Rafael.Menelau-Cruz, Eric.Granger\}@etsmtl.ca).}
\thanks{P.~Shamsolmoali is with the University of York, U.K. (Email: pshams55@gmail.com).}
}


\markboth{IEEE TRANSACTIONS ON INFORMATION FORENSICS AND SECURITY}%
{Shell \MakeLowercase{\textit{et al.}}: A Sample Article Using IEEEtran.cls for IEEE Journals}


\maketitle

\begin{abstract}
Visible-infrared person re-identification (V-I ReID) seeks to retrieve images of the same individual captured over a distributed network of RGB and IR sensors. Several V-I ReID approaches directly integrate the V and I modalities to represent images within a shared space. However, given the significant gap in the data distributions between V and I modalities, cross-modal V-I ReID remains challenging. A  solution is to involve a privileged intermediate space to bridge between modalities, but in practice, such data is not available and requires selecting or creating effective mechanisms for informative intermediate domains.
This paper introduces the Adaptive Generation of Privileged Intermediate Information (\AG) training approach to adapt and generate a virtual domain that bridges discriminative information between the V and I modalities. \AG enhances the training of a deep V-I ReID backbone by generating and then leveraging bridging privileged information without modifying the model in the inference phase. This information captures shared discriminative attributes that are not easily ascertainable for the model within individual V or I modalities. Towards this goal, a non-linear generative module is trained with adversarial objectives, transforming V attributes into intermediate spaces that also contain I features. This domain exhibits less domain shift relative to the I domain compared to the V domain. Meanwhile, the embedding module within \AG aims to extract discriminative modality-invariant features for both modalities by leveraging modality-free descriptors from generated images, making them a bridge between the main modalities.
Experiments conducted on challenging V-I ReID datasets indicate that \AG consistently increases matching accuracy without additional computational resources during inference. 

\end{abstract}


\section{Introduction} \label{sec1}

Person re-identification (ReID) involves matching images or videos of individuals of interest captured across multiple non-overlapping cameras. Current methods for person ReID, such as those based on deep Siamese networks \cite{TIFS4,fu2021unsupervised, sharma2021person, somers2023body}, mainly use visible images to train a feature embedding backbone for similarity matching metric learning losses. 

Although deep learning (DL) models have made significant progress in recent years, person ReID remains a challenging task due to the non-rigid structure of the human body and variations in capture conditions in real-world scenarios. These variations can include illumination variations, blurring of motion, resolution, pose, viewpoint, occlusion, and background clutter \cite{TIFS4, bhuiyan2020pose, mekhazni2020unsupervised}. Moreover, visible ReID systems cannot perform well in low-light environments (e.g., at night), thereby limiting their practical application.

\begin{figure*}[t!]
\centering
\includegraphics[width=\linewidth]{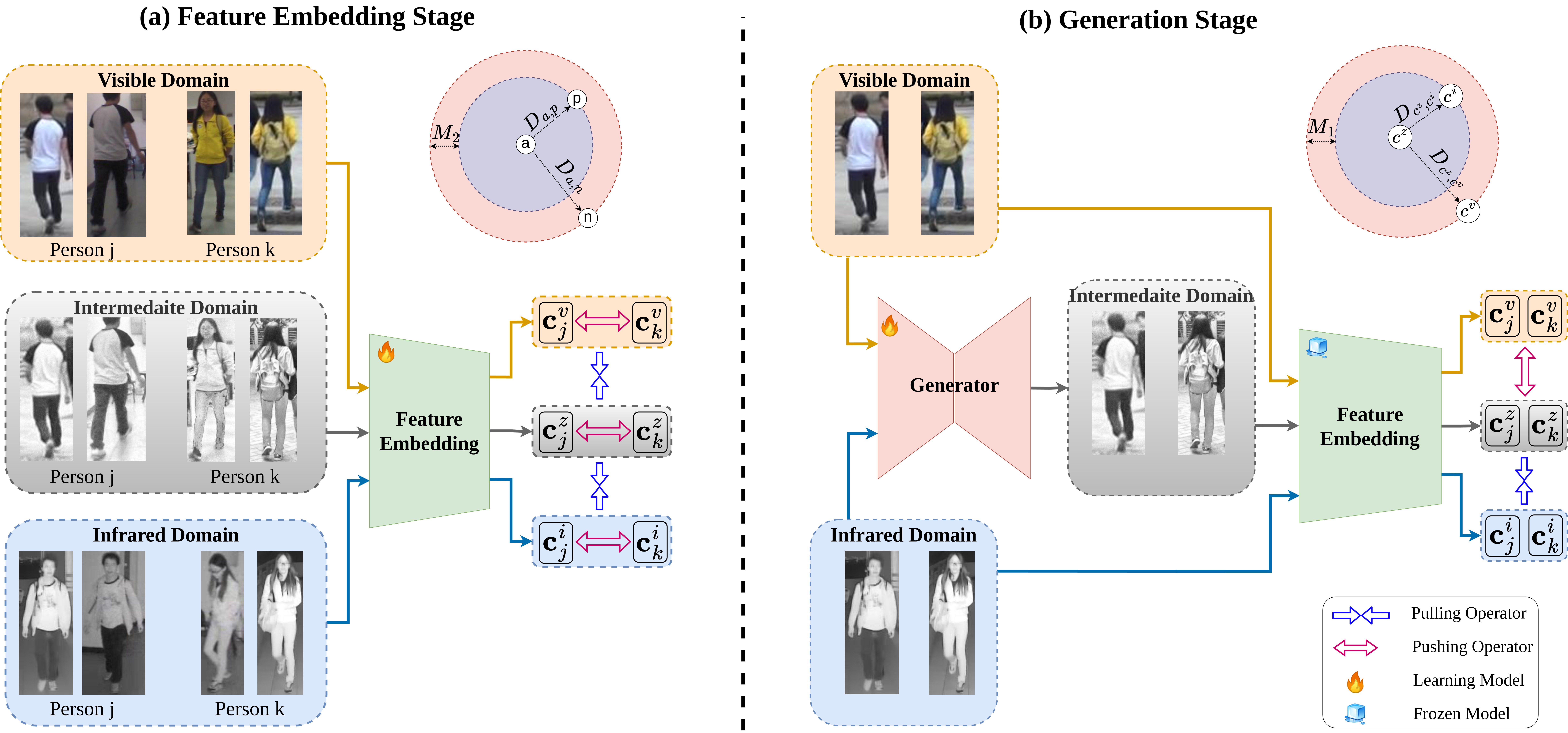}

\caption{A illustration of the training strategy with our \AG approach. The generated images are learned to be similar to the infrared modality by using adversarial objectives. (a) The feature embedding stage pushes the extracted features to approach the intermediate domain, while  (b) the generation stage transforms V images to an intermediate domain that approaches I images. The objective of feature embedding is to minimize the distance from the anchor to the positive ($D_{a,p}$) and maximize its distance from the negative samples ($D_{a,n}$). Let $c^v$, $c^z$, and $c^i$ be the center of the feature distributions for each class of V, intermediate, and I images. The generator tries to create a space in which the ID-aware features of each person are close to the infrared and visible at the same time. (To reach this goal and make the transferred images not be biased toward visible, we push them more toward infrared: $D_{c^z, c^i} < D_{c^z, c^v} + M_1 $).} 
\label{fig:intro}
\end{figure*}

Visible-infrared ReID (V-I ReID) is a variant of person ReID that involves matching individuals across RGB and IR cameras. V-I ReID is challenging since it requires matching individuals across different modalities with significant differences in appearance \cite{TIFS1,TIFS2,TIFS3,BMDG}. 
For robust cross-modal matching, a V-I ReID system must train a feature extractor that encodes pedestrians' images into feature information, which is interchangeable between V and I modalities. Then, during inference, images of people are matched across RGB and IR cameras based on these features. State-of-the-art methods for V-I ReID mainly focus on image-level \cite{Wang_2019_ICCV_AlignGAN,kniaz2018thermalgan} or feature-level alignment \cite{all-survey,Bi-Di_Center-Constrained,EDFL:journals/corr/abs-1907-09659,hetero-center,park2021learning}. Through the use of generative modules, the former aims to bridge the modality discrepancy by encoding images into feature space and then decoding and translating from different modalities to others. The latter obtains a modality-invariant embedding space by learning a shared model to represent images in feature space in which the discrepancy is minimized.
Since V and I modalities have distinct data distributions, which results in a large discrepancy between V and I images in appearance and statistics \cite{wu2017rgb}, the features extracted from such models tend to be biased on modality-specific information and not be proper for cross-modal matching.

Addressing this challenge involves considering an intermediate space that provides meaningful information about the primary modalities, guiding the model to focus on shared knowledge. However, the availability of such intermediate spaces is not guaranteed, and their creation or identification is not a straightforward task. Additionally, determining how to effectively utilize bridging information during both the training and testing stages poses a key aspect of this solution.

The Learning Under Privileged Information (LUPI) paradigm \cite{vapnik2009new} emerges as a solution aiming to leverage additional data available only during training. This approach has been adopted by various V-I ReID methodologies to bridge the gap between V and I images. For instance, works such as \cite{randLUPI,xModal,wei2021syncretic,HAT,zhang2021towards,Mixer} concentrate on exploring subspaces between the IR spectrum and RGB channels, creating intermediate domains to address the modality gap.
However, existing intermediate generative approaches fail to effectively adapt or generalize the intermediate information across V and I domains. The absence of specific objectives for generalizing intermediate images results in a limited ability to adapt to significant V-I domain shifts. To better capitalize on modality-invariant attributes, V-I ReID models necessitate appropriate intermediate images capable of bridging the substantial domain gap between I and V appearances. While utilizing intermediate images during testing enhances matching accuracy~\cite{kniaz2018thermalgan, Wang_2019_ICCV_AlignGAN}, it introduces complexity and may not be ideal for real-time applications.


This paper introduces an approach for the Adaptive Generation of Privileged Intermediate Information (\AG approach) that trains a backbone model for V-I person ReID by adapting a new intermediate space that connects V and I modalities. Our proposed \AG training approach comprises three key modules: a feature embedding backbone, a generator, and an ID-modality discriminator. Since each modality can be seen as a domain with its data distribution, we propose treating each modality as a separate domain. Then, our \AG aim is to generate an intermediate domain with similar attributes to the source V and I domains. This intermediate space enables the model to learn and extract common attributes between the two modalities. At inference time, the model can match the same person across modalities by representing them in a non-modality-specific feature space without relying on intermediate information. To achieve this, our framework adapts intermediate images to have the same content (body pose, background, and texture) as V images while sharing the style information (heat or temperature information) with I images. The \AG component provides an intermediate step for the feature embedding backbone, helping to reduce modality-specific information in the extracted features.

Given the absence of supervised information to generate the intermediate domain containing ID-aware information and bridging modality attributes, we propose an end-to-end approach with adversarial objectives w.r.t. the ultimate goal of V-I ReID, which is extracting unique and modality-free descriptors for each person. To this end, the generative module is encouraged to transform the V images to an intermediate domain to minimize the domain gap between such virtual and I images. Consequently, the generative module must focus on ID-aware attributes to avoid the intermediate images missing the discriminative information about the person in the V images. To reach this goal, the ID-modality discriminator is proposed, which seeks to reinforce the generator by discriminating modalities from ID-related features. Such property in the discriminator forces the generator to focus on ID-aware information. The feature embedding module simultaneously works to minimize the gap between the generated and V images, as shown in the left side of Fig. \ref{fig:intro}, while the generator's objective is to maximize this distance (right side of \ref{fig:intro}). This collective approach ensures the effective generation of intermediate domain images that capture both ID-aware information and modality attributes. 

In other words, \AG tackles the significant gap between modalities by generating and utilizing bridging information that links the two domains. Since such information is not available all the time or hard to benefit from, our \AG training approach is on two stages: (a) Finding a privileged linking domain between two modalities and (b) making the model robust against modality change by leveraging the intermediate domain as a clue for containing shared attributes among two modalities.  In comparison with \cite{MUN,idm} that are creating an intermediate domain in the latent feature space, our intermediate domain is in the input space. Therefore, our generative module can be used in other applications to generate intermediate images between visible and infrared and even in other V-I person ReID.
In Fig. \ref{fig:intro}, the generative module is depicted as trained to generate images with features closer to I and farther from V, enforcing a margin $M_1$ between the distances from the center of the feature vector representing images of each person in the intermediate space to V ($D_{c^z, c^v}$) and I ($D_{c^z, c^i}$). Simultaneously, the feature embedding stage ensures close similarity between the feature vectors of each person in the visible, intermediate, and infrared modalities.

\textbf{Our main contributions are summarized as follows:}
    (1) An adversarial training strategy called \AG is proposed to reduce the V and I domain gaps by adopting a generalized intermediate domain as privileged information. 
    (2) By analyzing the mutual information between intermediate space and V-I person ReID objectives, adversarial losses are proposed between the ID-modality discrimination, feature embedding, and generation modules to encourage the transformed images to incorporate discriminative ID information while reducing the influence of domain-specific attributes. Such losses make the modules play a min-max game between each other, resulting in a balanced and informative intermediate space.
    (3) An extensive set of experiments on the challenging SYSU-MM01 \cite{SYSU} and RegDB \cite{regDB} datasets indicate that our proposed \AG not only outperforms state-of-art methods for cross-modal V-I person ReID but can also be seamlessly integrated into other methods to enhance their accuracy without imposing additional computational overhead during inference. In other words, our framework is agnostic and can be applied to other image retrieval applications to improve their performance. 


\section{Related Work}

\noindent \textbf{(a) Visible-Infrared Person ReID:} 

This paper focuses on V-I ReID, where deep backbone models are trained to match individuals between V and I cameras using a labeled set of V and I images captured using RGB and IR cameras, respectively. Then, during inference, different individuals are matched across camera modalities, that is, matching query I images to gallery V images, or vice versa \cite{chen2021neural,cm-gan,fu2021cm,adv-modal,park2021learning,xu2021cross,HCML,hetero-center}. 
Existing DL models used for V-I ReID can be divided into Generative and Representative models. The former seeks to map raw data between two modalities to reduce modality discrepancies at the image level, while the latter focuses on finding a suitable discriminant representation by reducing the cross-modality gap at the feature level. 


Some generative models, like ThermalGAN \cite{kniaz2018thermalgan} or AlignGAN \cite{Wang_2019_ICCV_AlignGAN} employed autoencoders to translate V and I images and found a common feature space through the generation process. D\textsuperscript{2}RL \cite{D2RL} generated fake-real paired images from synthetic V and I images and encoded them into a common feature space for feature representation and matching. Other models \cite{HI-CMD,JSIA-paired-images1,paired-images2} relied on GANs to create synthetic images conditioned by a modality-invariant representation. However, these methods are complex at inference time, requiring image generation at test time. Moreover, they do not necessarily produce high-quality images that are capable of bridging the gap between modalities.


Representation approaches focus on training a backbone model to learn a discriminant representation that captures the essential features of each modality while exploiting their shared information \cite{HCML,cm-gan,Bi-Di_Center-Constrained,DZP,EDFL:journals/corr/abs-1907-09659,all-survey} or exploiting self-similarity \cite{temp2}. In \cite{all-survey,HAT,CAJ,jiang2020crossGranularity,park2021learning,JFLN}, authors only used global features combined through multimodal fusion.  
The absence of local information limits their accuracy in more challenging cases. To improve robustness, recent methods combined global information with local part-based features \cite{DDAG,cmSSFT,part1,wu2021Nuances,SAAI,PartMix23,BMDG} or fused features of different resolutions similar to LRAR \cite{temp1}, which used two resolution-adaptive mechanisms. For example, \cite{DDAG,cmSSFT} divided the spatial feature map into fixed horizontal sections and applied a weighted-part aggregation.
Using the weighted mechanism, these methods could dynamically pay attention to the relative importance of different parts to improve their discrimination power.
\cite{SAAI,wu2021Nuances, CATA} dynamically detected regions in the spatial feature map to address the misalignment of fixed horizontally divided body parts.


In contrast, the proposed \AG uses a generative model only during the training phase to generate the privileged information to help the training process by bridging the modality gap. This makes our proposed method computationally efficient, relying only on the feature embedding backbone during inference.

\noindent \textbf{(b) Intermediate Modality in V-I ReID:}
Given the significant gap between V and I distributions, cross-modal matching is a challenging problem. One possible approach is to use an intermediate modality to help the model find common semantics between these two modalities. For example, the model \cite{DZP} relied on gray images instead of colored images. However, using grayscale alone limited the amount of valuable information from V images by not containing the V spectrum information. Therefore, Fan et al. \cite{Fan2020CrossSpectrumDP} combined the information from the three channels (V, I, and grayscale) and \cite{HAT} used grayscale images as an intermediate modality during the training process, while in \cite{CAJ,liu2021sfanet, DDAG2} augmented images were created by randomly selecting one value from the RGB color as the intermediate modality. Alehdaghi et al. \cite{randLUPI} introduced a novel approach involving the use of random linear combinations of RGB colors. This technique served as an effective bridge, facilitating the extraction of color-independent features and ensuring that the derived attributes are not only devoid of color bias but also contain characteristics commonly found in infrared images. \cite{DDAG2} performed the channel-level interaction by randomly determining whether to keep the original RGB image or perform a random channel selection. However, these methods made limited use of intermediate domains by searching for them in a small search space.

To address this limitation, methods in \cite{xModal,wu2017rgb,zhang2021towards,TIFS1} used generative models to create the wider search space for intermediate modality. For example, \cite{xModal} used a single convolution layer to transform V images into a new modality without using the I or identity information. \cite{wu2017rgb} proposed a pixel-to-pixel feature fusion operation on V and I images to build the synthetic images, and \cite{zhang2021towards} introduced a shallow auto-encoder to generate intermediate images from both modalities. The images generated by these methods were biased toward the source modality.
Lu et al. \cite{TIFS1} designed a channel interactive generator to generate confused modality images to address this problem. 
While such methods enhance performance, they ignore the intricate V-I image relationship and fail to take into account domain gaps between id-discriminative information during the generation of intermediate domains. In contrast, our approach focuses on minimizing domain shifts between visible and infrared to identify the most informative intermediate space. It also employs an ID-modality discriminator to focus on the person's identity, allowing the generative model to robustly create persons’ images from one modality into an intermediate while preserving the discriminative attributes of that person. Thus, the quality of the intermediate image and the id-aware information is significantly improved.

\noindent \textbf{(c) Learning Under Privileged Information:}
Supervised learning is based on labeled data to optimize an appropriate decision function, minimizing the generalization error. The learner's function typically converges rapidly to the optimal solution for simple learning tasks with a large amount of training data. However, the convergence rate becomes slow for more challenging tasks and larger optimization spaces \cite{LUPI}. The LUPI approach was initially proposed to accelerate the convergence of SVMs by incorporating prior knowledge into the learning process. In \cite{hoffman2016learning} used a hallucination network to introduce prior knowledge into an additional stream in CNN to extract depth-related attributes. Beyond seeking faster convergence, LUPI has been successfully applied in various applications to improve performance \cite{choi2017learning,kampffmeyer2018urban,kumar2021improved,lezama2017not,saputra2020deeptio}. In \cite{pande2019adversarial}, a GAN-based model is proposed to generate prior knowledge at the testing time, and Crasto et al. \cite{crasto2019mars} distilled this knowledge from privileged information to the V branch. Although the use of PI improves model performance, such proper and supervised information is unavailable or difficult to leverage. 
In this study, our focus lies in implementing the LUPI paradigm for V-I ReID, aiming to identify suitable and informative Privileged Information (PI) during training. To achieve this objective, we propose a generative model designed to dynamically transform modalities into an intermediate space, serving as privileged information exclusively during the training phase through the use of adversarial losses.

\section{Proposed Method}
\label{sec:method}
\begin{figure*}[ht!]
\centering
    \begin{minipage}{0.67\textwidth} 
        \centering
        \subfloat[Overal training Architecture]{{
        \includegraphics[width=\linewidth]{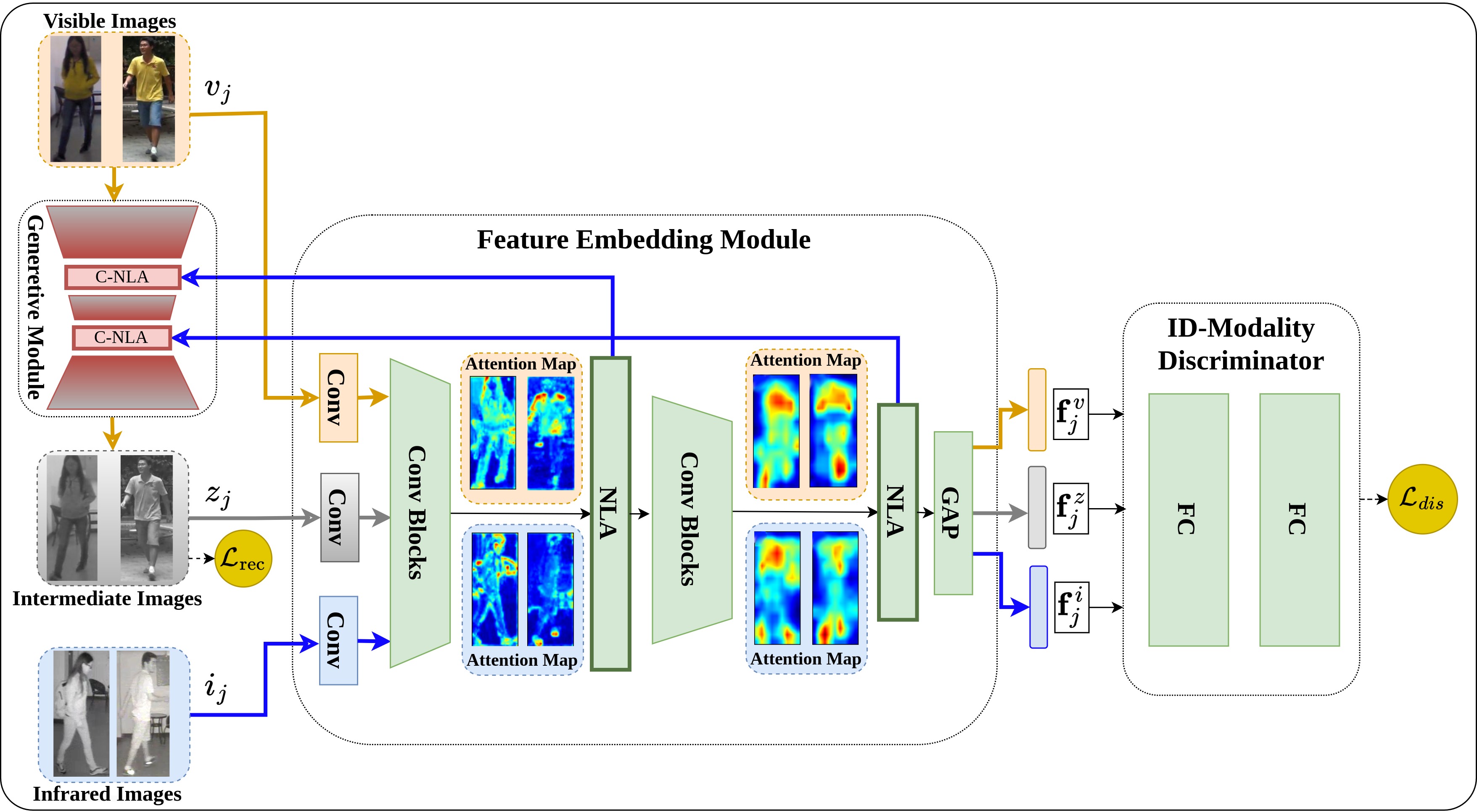}
        }}
    \end{minipage}
    \hfill
    \begin{minipage}{0.3\textwidth} 
        \centering
        \subfloat[Non-Local Attention (NLA) \cite{all-survey}]{
        \includegraphics[width=\linewidth]{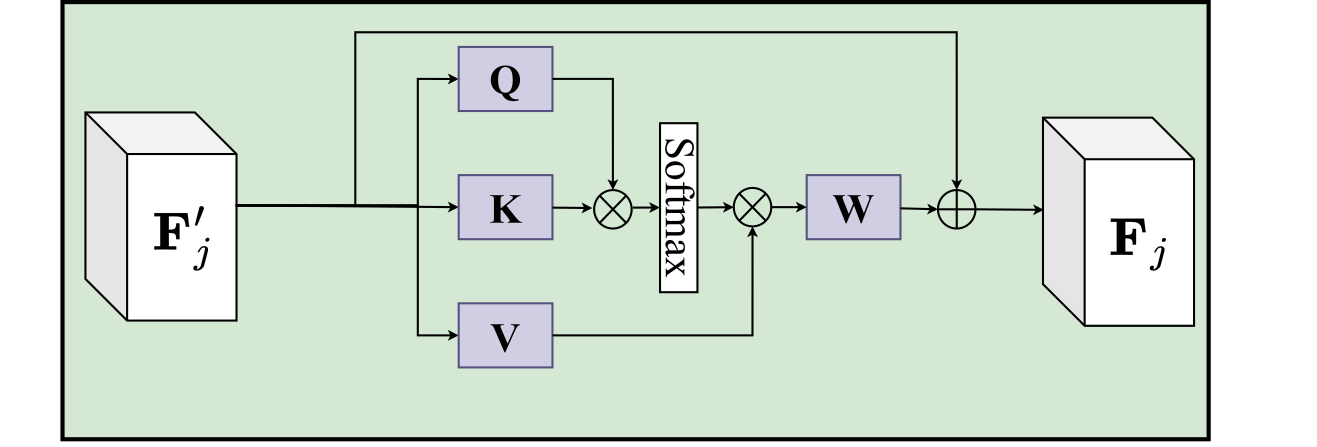}
        } \\
        \vspace{1.45cm} 
        \subfloat[Cross Local Attention (CLA)]{
        \includegraphics[width=\linewidth]{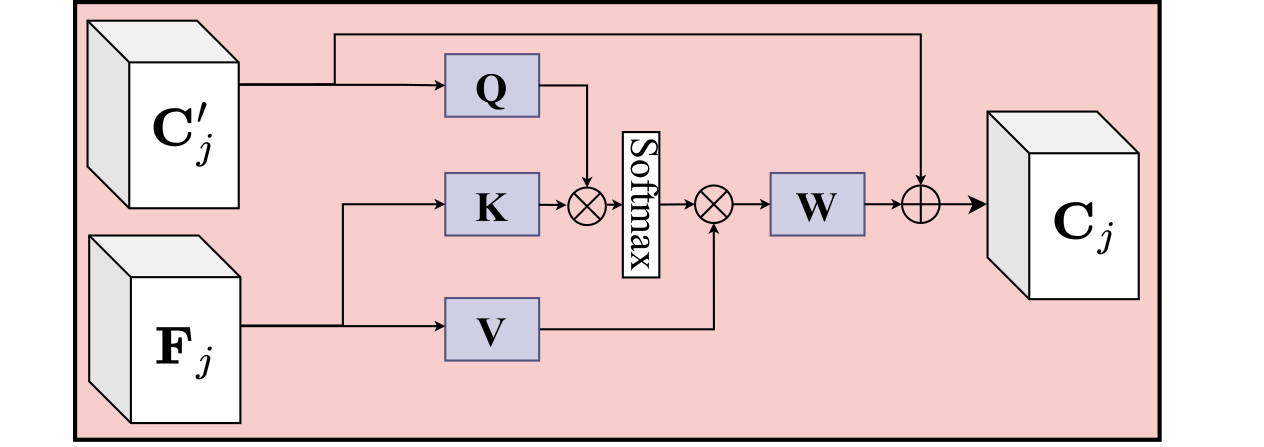}
        }
        
    \end{minipage}
\vspace{-0.3cm}
\caption{Block diagram of our proposed \AG training strategy for V-I person ReID. It takes advantage of the adaptive intermediate domain to reduce the distribution gap between the V and I domains. The feature embedding module seeks to minimize the distance between the two original (I and V) modalities by using the auxiliary generated intermediate domain.
The ID-modality discriminator learns to detect the modality from ID-aware features, and the generation process aims to transform V images to infrared through adversarial training with the discriminator.}
\label{fig:Model_triModal}
\end{figure*}

Our \AG approach generates a virtual PI domain that bridges V and I modalities during training and then leverages it to learn modality-invariant features.
Fig. \ref{fig:Model_triModal}(a) shows the overall \AG training architecture, comprising three modules: (1) The generator module that transforms V images into intermediate ones by incorporating infrared features. These intermediate images help the model simulate an infrared version of the visible images, allowing it to focus on non-color-specific attributes from the visible modality. Since such images have exact pose information of individuals and only lose the colors, the model accesses one intermediate step before seeing infrared images with larger shifts (different pose, background, and modality).
(2) ID-modality discriminator that identifies modality from ID-aware features to avoid the intermediate images being over-focused on the background, which is likely to be dominant for modality transformation,
and (3) A feature embedding module that extracts shared features from V and I images, guided by the proposed bridging losses to effectively minimize modality discrepancies.

During training, the ID-modality discriminator is taught to classify individuals in the intermediate and V images, considering them as belonging to the same class but distinct from the I images. This is achieved through the use of identity discriminative features extracted by the backbone. The primary goal is to guide the generative module in prioritizing the capture of personal attributes aligned with the characteristics associated with modality-specific information.
Concurrently, in the generation learning stage, the generator aims to produce intermediate images that the discriminator would adversarially classify as belonging to the I modality. This iterative process leads to the generated images becoming more similar to the I images, thereby enhancing the model's ability to extract common features from both V and I modalities.


\noindent \textbf{Preliminaries:} 
Let us assume a multimodal dataset for V-I ReID problems that is represented by $\mathcal{S} = \{ \mathcal{V}, \mathcal{I}\}$, where $\mathcal{V}=\{v_j\}_1^n$ are V and $\mathcal{I}=\{i_j\}_{n+1}^{n+m}$ are I images of $N$ different persons, and $\mathcal{Y} = \{y_j\}_1^{n+m}$ contains their identity for the training process. Metric losses $\mathcal{L}$ are typically defined based on distances $D$ between feature vectors $\mathbf{f}^v_j = \varphi(v_j;\Theta)$ of V images and feature vectors $\mathbf{f}^i_j = \varphi(i_j;\Theta)$ of I images for every two distinct individuals, as:
\begin{equation}
\label{eq:problem}
    \mathcal{L} = \sum_{j=1}^{n+m}{ \mathcal{C}(D(\mathbf{f}^v_j, \mathbf{f}^i_j) , D(\mathbf{f}^v_j, \mathbf{f}^v_k), y_j, y_k)},
    \vspace{-0.1cm}
\end{equation}
where $\mathcal{C}$ is metric loss objectives which can be triplet-loss \cite{chopra2005learning} or contrastive loss \cite{schroff2015facenet}, $y_j$ and $y_k$ represent the identity of person in images $j$ and $k$, and $\mathbf{f}^v_j, \mathbf{f}^i_k \in \mathbb{R}^d$, where $d$ the features dimension.


\subsection{LUPI Framework}

Given that the generated images are unavailable at test time, we formulate our proposed learning strategy according to the LUPI paradigm. With LUPI \cite{lambert2018deep}, the DL model leverages privileged information (PI) exclusively during the training phase, while the main information remains accessible in both the training and testing phases. For example, in cross-modal matching, the model typically relies on the V and I modalities during the training phase to leverage shared discriminative features between inputs. However, during the inference phase, the model only inputs the query image from one modality (V or I). Also, during the training phase, the model can use privileged bridging information between the two modalities to address the significant domain discrepancies between V and I images \cite{LUPI}. 
By incorporating the intermediate modality ($Z$), \AG reduces the modality gap between V and I, enabling the model to generalize more effectively and improve its ability to transfer knowledge across both modalities.
\subsection{Mutual Information Analysis}

We formulate the bridging characteristics objectives of the privileged intermediate information from the mutual information perspective, and demonstrate that jointly learning id-aware and modality-free objectives achieves a conditional mutual information maximization between intermediate images and identity space while discarding the modality style information:
\begin{equation}
\label{eq:mi_main}
    \max_{Z}  \quad \MI(Z;Y|M),
\end{equation}
where $Y$ is the identity of person input images, $M$ is the modality label, and $\MI(.;.)$ is the mutual information between two random variable. For solving Eq. (\ref{eq:mi_main}), we can expand it as follows: 
\begin{equation}
    \MI(Z;Y|M) =  \overbrace{\MI(Z;Y)}^{\textit{ID-Aware}} - \overbrace{\MI(Z;Y;M)}^{\textit{ID-Modality-Aware}}. 
    \label{eq:mi2}
\end{equation}

To maximize this, the first term (ID-aware) should be maximized and the second term (ID-Modality-Aware) minimized. 
To minimize the ID-Modality-Aware part for intermediate images, an adversarial training approach is proposed between the ID-Modality Discriminator and the generative module, in which the former maximizes that part by looking for id-related information specific to the modality. At the same time, the latter tries to create $Z$ images that fool the discriminator, thereby minimizing this term.

\subsection{ID-Modality Discriminator}

For the generator to create intermediate images that contain id-related information, we propose a module to discriminate modality along with the identities of each individual in images.
To achieve this, we employ an MLP network as our ID-modality Discriminator to classify modality types (Infrared or Visible) from features extracted from the shared embedding backbone. The generator's objective is to deceive the discriminator by creating a modality that appears false for a specific identity.

To implement this idea, we doubled the label space of each identity to train the discriminator to differentiate between the same identity in the V and I modalities:
\begin{equation}
y'_{j} = \begin{cases}
2y_j &\text{V and Intermediate modalities}\\
2y_j + 1 &\text{I modality}
\end{cases} 
\label{eq:disc_label}
\end{equation}
where $y_j$ is the original label.

For training the ID-modality discriminator, the objective is defined as: 
\begin{equation}
    \mathcal{L}_\text{dis} = \sum_{l=1}^{2N} -y'_{j,l}\log(p'_{j,l}),
    \label{eq:all_d}
    \vspace{-0.1cm}
\end{equation}
where $p'_{j,c}$ is the predicted probability observation $j$ of class $c$ by the discriminator $\mathcal{D}$.

Fig. \ref{fig:disc} shows the main difference between the general modality discriminator and our ID-modality discriminator. Each training sample is a shape labeled with the target label. Other SOTA intermediate generative approaches use ordinary binary discriminators that are trained to identify only modality (V and intermediate as the same), disregarding the person's identity. In contrast, by proposing new label spaces that are related to the identities of the discriminator, it would be able to detect V/I modalities from the ID-aware attributes of samples during training. 
\begin{figure}[t]
\begin{center}
\includegraphics[width=0.95\linewidth]{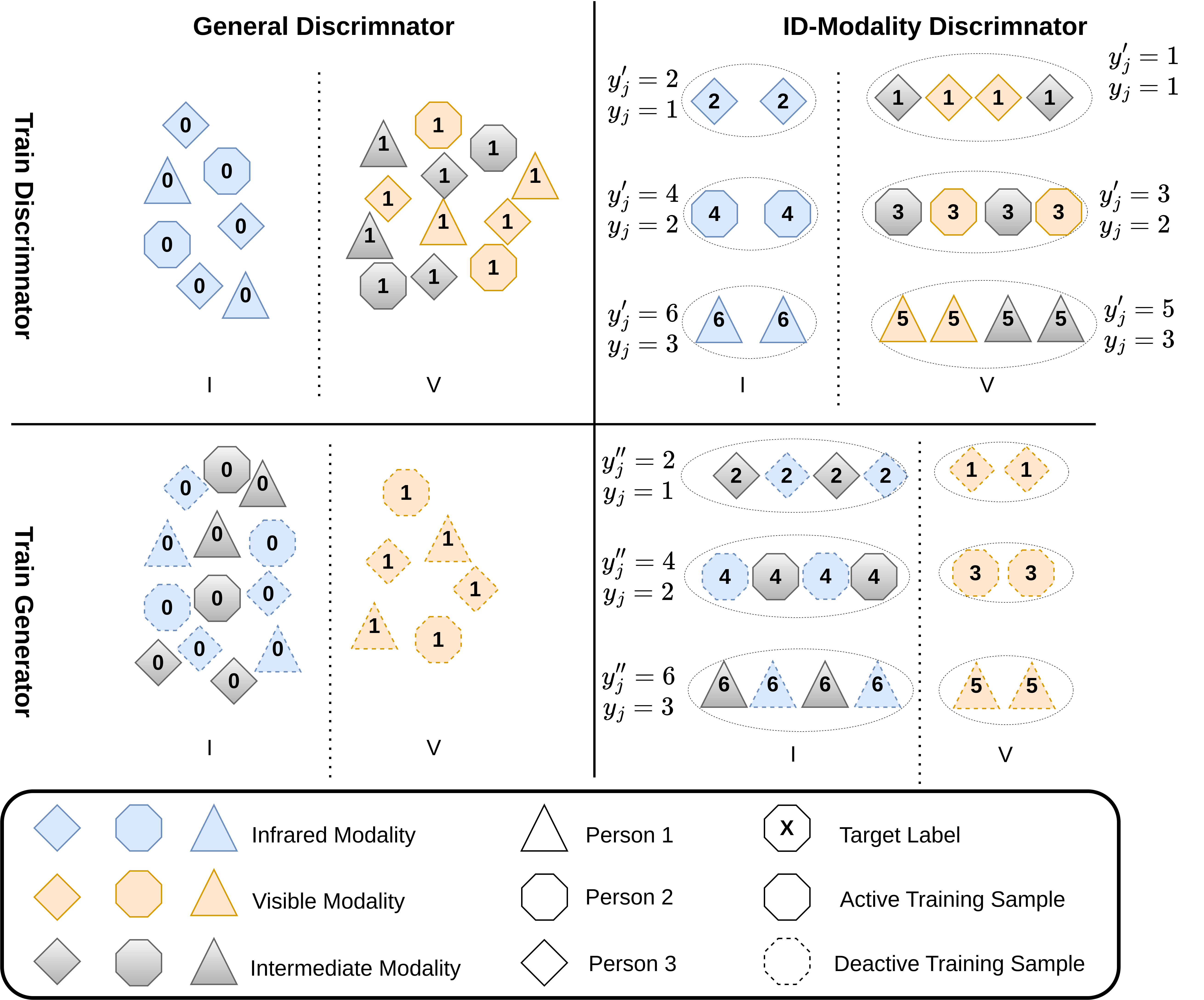}
\end{center}
   \caption{ID-Modality discriminator vs. general discriminator. Our label space is doubled to account for identity (differentiate between individuals in each modality). 
   }
\label{fig:disc}
\end{figure}

\subsection{Generative Module}
To generate the intermediate images, an adaptive and unsupervised intermediate generative model $\mathcal{G}$ is proposed to transfer visible images to intermediate modality in an end-to-end fashion regarding Eq. (\ref{eq:mi_main}). The reason behind using visible images is that the visible images contain more modality-specific id-discriminative information \cite{SIMcm} and need to be filtered to remove this biased information.
To have information from both modalities in the generation process, the V images, at first, are encoded as features vectors and fused with I features $F^i_{j}$ by the cross non-local attention \cite{non-local-net}, and then the decoder reconstructs the image, $z_j=\mathcal{G}(v_i, F^i_{j})$. 
To allow this module to focus on ID-aware information in intermediate images and maximize the $\MI(Z;Y)$, the cross-entropy between identity space and intermediate images must be minimized (we proved this in Appendix \ref{sec:prof1}). So the $\mathcal{L}_\text{id}$ which is a cross-entropy loss function between the features extracted from intermediate images ($\mathbf{f}^z_j$) and identity label $y_j$ is proposed to be minimized. Also, to filter them from any modality-specific information, these images should maximize the $\mathcal{L}_\text{dis}$ when they are fed to the ID-modality discriminator. To make this, we need a new label to fool the discriminator, and a label that specifies the same identity but in a different modality is chosen:
\vspace{-6pt}
\begin{equation}
\label{eq:gen_label}
    y''_{j} = 2y_j + 1.
\end{equation}
The reason is that since this discriminator is not binary, every label except the ground truth with which the discriminator is trained can be selected as an opposite label. So, the proper label should be selected to avoid interrupting the id-discriminative information in the created images.

Also, since the Z images are created by transferring from V to intermediate space, they tend to be biased toward visible space. To tackle this issue, we pushed their feature more toward infrared than visible features. This adversarial objective can be defined as:  
\begin{equation}
\begin{aligned}
\mathcal{L}_\text{adv} &= \sum_{l=1}^{2N} -y''_{j,l}\log(p''_{j,l}) + \max(0, M_1 + D_{\mathbf{c}^z_j,\mathbf{c}^i_j} - D_{\mathbf{c}^z_j,\mathbf{c}^v_j}),
\end{aligned}
\label{eq:adv}
\end{equation}
where $p''_{j}$ is the predicted label $j$ from the discriminator, and $\mathbf{c}^z_j$, $\mathbf{c}^i_j$,$\mathbf{c}^v_j$ are the centers of feature embedding of generated I, and V images, respectively, for the person with identity $y_j$. The function $D$ measures the distance between two vectors.


To generate more realistic images, we also transfer I images into the intermediate domain, which provides a supervised reconstruction loss:
\begin{equation}
\mathcal{L}_{\text{rec}} = \vert \mathcal{G}(i_j, F^i_{j}) - i_j  \vert,
\end{equation}
and the total loss for the generative modules is:
\begin{equation}
\mathcal{L}_{\text{gan}} = \mathcal{L}_\text{rec} + \mathcal{L}_\text{idz} + \lambda_{\text{adv}} \mathcal{L}_\text{adv}.
\label{eq:all_g}
\end{equation}
\subsection{Feature Embedding Module}
The feature embedding module, $\mathcal{F}$, has three backbones with shared layers to extract features from V, intermediate, and I images. Modified versions of the triplet and color-free losses are proposed to encourage modality-invariant descriptor features and train the V-I ReID model while considering intermediate images. 

\noindent \textbf{(1) Color-Free Loss:}
Previous studies \cite{LUPI,HAT} have shown the advantages of feature invariance when extracting features from both original V and generated modality images Z. For feature representations that are robust to modality variations, we used the same color-independent loss over virtual and V images as defined in \cite{LUPI} for training the V and intermediate backbone:
\begin{equation}
\mathcal{L}_{\text{cf}} = \lVert \mathbf{f}^v_j - \mathbf{f}^z_j \rVert,
\end{equation}
where $\mathbf{f}^z_j$ is the feature vector of the intermediate images $z^\star_j$.

\begin{algorithm}[t]
\caption{Joint \AG Training Strategy of $\mathcal{F}$, $\mathcal{G}$ and $\mathcal{D}$} \label{alg:joint_learning}
\begin{algorithmic}
\Require $\mathcal{S} = \{ \mathcal{V}, \mathcal{I}\}$ as training data
\While{all batches are not selected}
    \State $v_j, i_j \Leftarrow \textbf{batchSampler}(b,p)$  
    \State $i^\star_j \Leftarrow \mathcal{G}(i_j, F^i_{j})$ \Comment{reconstruct I} 
    \State $z^\star_j \Leftarrow \mathcal{G}(v_j,F^i_{j})$ \Comment{generate intermediate with detached gradient style feature $F^i_{j}$} 
    \State unfreeze $\mathcal{F}$ and $\mathcal{D}$
    \State update $\Theta_\mathcal{F}$ by optimizing Eq. (\ref{eq:all_f}) \Comment{detach gradient from $z_j$ to not modify $\mathcal{G}$ }
    \State update $\Theta_\mathcal{D}$ by optimizing Eq. (\ref{eq:all_d}) while freezing $\mathcal{F}$
    \State update $\Theta_\mathcal{G}$ by optimizing Eq. (\ref{eq:all_g}) while freezing $\mathcal{D}$
\EndWhile
\end{algorithmic}
\end{algorithm}
\noindent \textbf{(2) Intermediate Dual Triplet Loss:}
Since the role of intermediate space is bridging between modalities, we use extracted features from intermediate images as the middle between them and apply two separate triplet losses in dual mode:  
\begin{equation} \label{eq:gray-triplet}
    \mathcal{L}_{\text{dual}} = \mathcal{L}_{\text{tri}}(\mathbf{f}_{a\in \mathcal{V}},\mathbf{f}_{p \in \mathcal{I}},\mathbf{f}_{n \in \mathcal{Z}}) + \mathcal{L}_{\text{tri}}(\mathbf{f}_{a\in \mathcal{I}},\mathbf{f}_{p \in \mathcal{Z}},\mathbf{f}_{n \in \mathcal{V}}),
\end{equation}
\noindent where $\mathcal{L}_{\text{tri}}$) \cite{DDAG,all-survey} is common soft-margin triplet loss\cite{DDAG,all-survey} with margin $M_2$,
$a$, $p$, and $n$ indicate the anchor, positive, and negative samples, respectively. The superscripts $i$, $v$, and $z$ select V, I, and intermediate features, respectively.

The total loss proposed to train the feature embedding backbone is:  
\begin{equation}
\mathcal{L}_\text{f} = \mathcal{L}_{\text{id}} + \mathcal{L}_{\text{dual}} + \lambda_{\text{cf}} \mathcal{L}_{\text{cf}}.
\label{eq:all_f}
\end{equation}

\noindent where $\mathcal{L}_{\text{id}}$ is the cross-entropy loss to allow identity-discriminative features \cite{all-survey}.

\noindent \textbf{Joint Learning:}
Each batch of data contains $b$ person with $p$ positive images from the V and I domains. After extracting the embedding features, the model generates intermediate images as described in Algorithm \ref{alg:joint_learning}. Our joint training strategy optimizes the generator $\mathcal{G}$ by freezing the $\mathcal{F}$ and $\mathcal{D}$ parameters and also detaches the gradients from the generated images when the parameters $\mathcal{F}$ are updated.

\section{Results and Discussion}\label{sec2}

\begin{table*}[!t]
\centering
\resizebox{\linewidth}{!}{%
\begin{tabular}{|c|l|c||c|c|c||c|c|c||c|c|c||c|c|c|} 
\hline
\multicolumn{3}{|c||}{\multirow{2}{*}{\textbf{Family}}} & \multicolumn{6}{c||}{\textbf{SYSU-MM01}} & \multicolumn{6}{c|}{\textbf{RegDB}} \\
\cline{4-15}
\multicolumn{3}{|c||}{}& \multicolumn{3}{c||}{All Search}  & \multicolumn{3}{c||}{Indoor Search} & \multicolumn{3}{c||}{Visible $\rightarrow$ Infrared}  & \multicolumn{3}{|c|}{Infrared $\rightarrow$ Visible}\\ \hline
\multicolumn{2}{|c|}{\textbf{Method}} & Venue   & \textbf{R1} & \textbf{R10} & \textbf{mAP} & \textbf{R1} & \textbf{R10} & \textbf{mAP}& \textbf{R1} & \textbf{R10} & \textbf{mAP} & \textbf{R1} & \textbf{R10} & \textbf{mAP} \\ \hline \hline
\multirow{5}{*}{\rotatebox[origin=c]{90}{Generative}}   & AGAN  \cite{Wang_2019_ICCV_AlignGAN}&ICCV19  & 42.4   & 85.0 & 40.7 & 45.9 & 92.7 & 45.3 & 57.9   & - & 53.6 & 56.3 & - & 53.4 \\ 

& D\textsuperscript{2}RL  \cite{D2RL}&CVPR19  & 28.9 & 70.6 &  29.2 & - & - & - & 43.4 & 66.1 &  44.1 & - & - & -\\ 
 & JSIA \cite{JSIA-paired-images1}&AAAI20  & 45.01  &  85.7 & 29.5  & 43.8 & 86.2 & 52.9 & 48.5  &  - & 49.3  & 48.1 & - & 48.9\\ 
 & Hi-CMD \cite{HI-CMD}&CVPR20  & 34.94   & 77.58 & 35.94 & - & - & - & 86.39 & - & 66.04 & - & - & - \\
 & TS-GAN \cite{zhang2021rgb}&PRL21   & 58.3   & 87.8   & 55.1 & 62.1 & 90.8 & 71.3 & 68.2   & -   & 69.4 & - & - & -  \\ 
\hline

\multirow{8}{*}{\rotatebox[origin=c]{90}{Representation}} 
& AGW\cite{all-survey} &TPAMI20 & 47.50 & - & 47.65  & 54.17 & - & 62.97 & 70.05 & - & 50.19  & 70.49 & 87.21 & 65.90\\  [-0.4ex]
&DDAG \cite{DDAG} & ECCV20 & 54.74 & 90.39 & 53.02 & 61.02& 94.06 & 67.98& 69.34& 86.19& 63.46& 68.06&85.15& 61.80 \\
 
& SSFT \cite{cmSSFT} & CVPR20  &  63.4 & 91.2 & 62.0  & 70.50 & 94.90 & 72.60 &  71.0 & - & 71.7  & - & - & -  \\ 
& SIM \cite{SIMcm}&IJCAI20 &  56.93 & - & 60.88 & - & - & - &  74.47 & - & 75.29 & 75.24 & - & 78.30\\ 
& MPANet \cite{wu2021Nuances}&CVPR22  & 70.58 & 96.21  & 68.24 & 76.74 & 98.21 & 80.95 & 83.70  & - & 80.9 & 82.8 & - & 80.7\\ 
&DTRM \cite{DDAG2} & TIFS22 & 63.03 & 93.82 & 58.63 & 66.35& 95.58 & 71.76& 79.09& 92.25& 70.09& 78.02&91.75& 69.56 \\
& FTMI \cite{sun2023visible}&MVA23 & 60.5 & 90.5 & 57.3  & - & - & - & 79.00 & 91.10 & 73.60  & 78.8 & 91.3 & 73.7  \\ 
& RAPV-T \cite{zeng2023random}&ESA23   & 63.97 & 95.30  & 62.33 & 69.00 & 97.39 & 75.41 & 86.81 & 95.81  & 81.02 & 86.60 & 96.14 & 80.52\\
 \hline
\multirow{7}{*}{\rotatebox[origin=c]{90}{Intermediate}}   &  HAT\cite{HAT}&TIFS20  & 55.29 & 92.14  & 53.89 & - & - & -& 71.83 & 87.16  & 67.56 & 70.02 & 86.45 & 66.30\\ [-0.4ex]
&  CAJ \cite{CAJ}&ICCV21   & 69.88 & -  & 66.89  & 76.26 & 97.88 & 80.37& 85.03 & 95.49  & 79.14  & 84.75 & 95.33 & 77.82\\ 
& RPIG \cite{randLUPI}&ECCVw22  & 71.08 & 96.42  & 67.56 & 82.35 & 98.30 & 82.73 & 87.95 & 98.3 & 82.73 & 86.80 & 96.02 & 81.26\\
& MMN \cite{zhang2021towards}&ICM21  & 70.60 & 96.20  & 66.90 & 76.20 & \textbf{99.30} & 79.60& \textbf{91.60} & 97.70  & \textbf{84.10} & 87.50 & 96.00 & 80.50  \\
& SMCL\cite{wei2021syncretic}& ICCV21& 67.39 & 92.84 & 61.78  & 68.84 & 96.55 & 75.56& 83.93 & -  & 79.83 & 83.05 & - & 78.57\\
& G2DA \cite{WAN2023109150} & PR23  &  63.94 & 93.34 & 60.73 & 71.06 & 97.31 & 76.01 & - & - & - & - & - & -   \\ 
&MCBD \cite{TIFS1} & TIFS23 & 71.63 & 94.98 & 67.28 & 79.44& 98.32 & 79.85& 90.10& 97.09& 82.73& 87.66 & 96.68& 81.52 \\
\hline
\multicolumn{2}{|c|}{\AG (ours)} & - & \textbf{72.23} & \textbf{97.04} & \textbf{70.58} & \textbf{83.45} & 98.62 & \textbf{84.25} & 89.03 & \textbf{98.19} & 83.89 & \textbf{87.91} & \textbf{97.15} & \textbf{83.04}\\
\hline
\end{tabular}}
\caption{Accuracy of the proposed \AG and state-of-the-art methods on the SYSU-MM01 (single-shot setting) and RegDB datasets. All numbers are percent.}
\label{tab:all-results}
\end{table*}
\begin{table}[!tb]
\centering
\resizebox{\linewidth}{!}{%

\begin{tabular}{|c|l||c|c||c|c|} 
\hline
\multicolumn{2}{|c||}{\multirow{2}{*}{\textbf{Family}}} & \multicolumn{4}{c|}{\textbf{LLCM}}  \\
\cline{3-6}
\multicolumn{2}{|c||}{}& \multicolumn{2}{c||}{V $\rightarrow$ I}  & \multicolumn{2}{c|}{I $\rightarrow$ V} \\ \hline
\multicolumn{1}{|c|}{\textbf{Method}} & Venue   & \textbf{R1} & \textbf{mAP} & \textbf{R1}  & \textbf{mAP}\\ \hline \hline 
DDAG\cite{DDAG} & ECCV20 & 40.3 & 48.4 & 48.0 & 52.3 \\
CAJ\cite{CAJ} & ICCV21 & 56.5 & 59.8 & 48.8 & 56.6 \\
RPIG\cite{randLUPI} & ECCVw22 & 57.8 & 61.1 & 50.5 & 58.2 \\
\hline 
\AG (ours) & - & \textbf{61.78}  & \textbf{65.08} & \textbf{56.47} & \textbf{62.79} \\
\hline
\end{tabular}}
\caption{
Accuracy of the proposed \AG and state-of-the-art methods on the Large LLCM Dataset. All numbers are percent.}

\label{tab:LLCM-results}
\end{table}

\subsection{Experimental Methodology}
\noindent \textbf{Datasets:} Cross-modal research has made extensive use of the challenging SYSU-MM01 \cite{SYSU}, RegDB \cite{regDB} and recently published LLCM \cite{LLCM} datasets. SYSU-MM01 is a large-scale dataset that contains more than 22K and 11K V, and I images that are not co-located belong to 491 individuals, respectively. These images were captured from four RGB and two near-infrared cameras, with 395 identities used for training and 96 for testing. The dataset has two evaluation modes, single-shot and multi-shot, based on the number of images in the gallery. The former considers only one random image per identity in the gallery, while the latter considers ten images per identity. The RegDB dataset comprises 4,120 co-located V-I images from 412 identities captured from a single camera. Ten trial configurations randomly divide the dataset into two identical sets (206 identities) for training and testing. Tests are conducted in two different ways, comparing I to V (query) and vice versa. An LLCM dataset consists of a large, low-light, cross-modality dataset that is divided into training and testing sets at a 2:1 ratio.

\noindent \textbf{Implementation Details:} Pre-trained ResNet50 \cite{resnet} with two non-local attention modules is used as our feature extractor, the same as the AGW model\cite{all-survey}. Also, concerning the analysis of different models and inputs in \ref{sec:exp_gen}, we choose VQ-VAE \cite{van2017neural} as the adaptor for V images into the intermediate domain. All inputs are resized to 288 by 144, then cropped randomly and filled with zero padding or mean pixels. SGD and ADAM optimizers were used for the optimization process of re-identification and generative modules, respectively. 
To perform the triplet loss, a minimum of two distinct individuals must be included in each mini-batch. So, we set $p=4$ different images of $b=8$ distinct persons from the dataset for each batch. $M_1=0.1$, $M_2=0.3$, $\lambda_{\text{adv}}=0.1$, and $\lambda_{cf}=10$ were set based on the analysis conducted in \ref{sec:exp_hyper}.

\noindent \textbf{Performance Measures:} We use Cumulative Matching Characteristics (CMC), and Mean Average Precision (mAP) as evaluation metrics. The CMC is rank-k accuracy, which determines the probability of a correct cross-modality person image appearing in the top-k retrieved results. In contrast, when multiple matching images are found in a gallery, mAP measures the image retrieval performance.

\subsection{Comparison with State-of-the-Art Methods}

Table \ref{tab:all-results} provides a comparison of our proposed \AG with state-of-the-art V-I ReID approaches. Our experiments show \AG outperforms these methods in various situations. Although our approach requires additional image generation during training, this process is not required during inference, making it suitable for real-time applications. Additionally, the model learns to robustly represent features across modalities. In particular, our model improves the R1 accuracy and mAP score on the large-scale SYSU-MM01 dataset by 24.73\% and 22.93\%, over AGW's \cite{all-survey} (our baseline). Compared to the second-best approach for the "indoor search" scenario,  \AG outperforms by a margin of 1.15\% R1 and 2.34\% mAP without adding complexity to the feature backbone. For the RegDB dataset, \AG outperforms by a margin of R1 and mAP by 1.08\% and 1.16\%, for "visible to thermal" ReID scenarios. Similar improvements are noticeable in the "thermal to visible" scenario. In addition, our approach has the advantage that it can be easily used in different cross-modal ReID models to mine PI only in the training phase without incurring overhead during testing.

The \AG model achieves improvements on the large and complex
LLCM dataset as shown in Table \ref{tab:LLCM-results}. Since this dataset was introduced recently, few papers reported their results, so we ran other competitors' approaches with published code on the LLCM dataset. 

\vspace{-4pt}
\subsection{Incorporating to Other State-of-the-arts Models}

The effectiveness of \AG is evident when integrated into state-of-the-art V-I ReID models \cite{DDAG, wu2021Nuances, shape-Erase23}. We executed the authors' provided code for each method, using their specified hyperparameters, both with and without AGPI$^2$. Table \ref{tab:baselines} highlights that \AG significantly improves performance when incorporated into these baseline models. For instance, integrating \AG into DDAG \cite{DDAG} and SEFL \cite{shape-Erase23} boosts mAP accuracy by 4\% and 1\%, respectively. Notably, SEFL  uses supervised body-shape information from modalities and also utilizes gray images as a bridging modality alongside visible and infrared images in training, while DDAG is trained solely on visible and infrared images. Therefore, our \AG framework showcases its bridging ability more effectively in DDAG compared to SEFL. 

\begin{table}[h!]
\centering
\caption{Accuracy of SOTA methods with \AG on the SYSU-MM01, testing under single-shot setting. Results were obtained by executing the author's code on PyTorch v1.12 on Ubuntu 22.04 with 40 GB NVIDIA A100 GPU.}
\label{tab:baselines}
\vspace{-0.05cm}

\begin{tabular}{|l||c|c|}
\hline
\textbf{Method}         & \textbf{R1 (\%)} & \textbf{mAP (\%)} \\ \hline \hline
DDAG  \cite{DDAG}               & 52.74         & 51.92            \\
DDAG with \AG                  & \textbf{57.38}          & \textbf{55.27}            \\ \hline
MPANet* \cite{wu2021Nuances}     & 70.58         & 68.24            \\
MPANet* with \AG                  & \textbf{72.85}          & \textbf{70.66}            \\ \hline
SEFL \cite{shape-Erase23}      & 74.05          & 69.21            \\
SEFL with \AG                  & \textbf{75.17}          & \textbf{71.43}            \\ \hline
\end{tabular}

\end{table}

\begin{table}[t]
\small
\centering

\begin{tabular}{|l||c|c||c|c|}
\hline 
\multirow{2}{*}{\textbf{Modality}}  & \multicolumn{2}{c||}{\textbf{SYSU-MM01}} & \multicolumn{2}{c|}{\textbf{RegDB}} \\ \cline{2-5}
                                    & I         & V         & I         & V       \\ \hline  \hline
\textbf{I}nfrared                   & 0         & 0.86      & 0         & 1.05  \\
\textbf{V}isual                     & 0.86      & 0         & 1.05      & 0      \\ \hline
Grayscale (HAT \cite{HAT})          & 0.68      & 0.31      & 0.93      & 0.36  \\ 
Random Comb. (RPIG \cite{randLUPI}) & 0.64      & 0.34      & 0.92      & 0.39  \\
MMN \cite{zhang2021towards} & 0.61      & 0.33      & 0.87      & 0.43  \\

\AG (ours)                     & \textbf{0.58}      & \textbf{0.43}      & \textbf{0.79}      & \textbf{0.52} \\ \hline

\end{tabular}
\caption{MMD between V, I, and different intermediate generation images on the SYSU-MM01 and RegDB datasets. For the I column, lower is better, and for the V, higher is better. 
}
\label{tab:MMD}

\end{table}
\vspace{-4pt}
\subsection{Domain Shift Between Modalities} \label{sec:domain_shift}
As we discussed earlier, the privileged intermediate domain bridges the main domains by sharing discriminative attributes that are common to both. Crucially, this intermediate domain must remain unbiased towards any particular domain and lie at a balanced point. To validate that our created virtual domain meets this criterion, we measured the MMD distance between the intermediate domain and both V and I domains that reflect a balanced intermediary position, ensuring an unbiased connection.

The kernel-based Maximum Mean Discrepancy (MMD) \cite{gretton2012kernelMMD} is used to measure the distance between deep features extracted by a ResNet50 model trained on ImageNet\cite{imagenet}. 
Table \ref{tab:MMD} shows this distance for V, I, and intermediate images for the SYSU-MM01 dataset. As described in Section \ref{sec:method}, the intermediate domain must bridge the main domains to each other. Hence, it needs to have a smaller discrepancy to V and I, i.e.,  $\text{MMD}(V,Z) \leq \text{MMD}(I,V)$ and $\text{MMD}(I,Z) \leq \text{MMD}(I,V)$. For example, on the SYSU-MM01 dataset, the discrepancy between the intermediate images generated using our \AG  method and I is 0.58, which is lower than the V-I discrepancy of 0.86. When compared to fixed gray images \cite{HAT}, which remain static and fail to adapt to the infrared spectrum, random gray images \cite{randLUPI} exhibit superior bridging capabilities. This advantage stems from their lack of bias towards specific color spectrums. In contrast, \AG images strike a balance between I and V modalities, leveraging adaptive information across color spectrums during generation.
This characteristic ensures that the intermediate space provides optimal information for the feature backbone to leverage common discriminative features between V and I images.
Also, for visual comparison, we show examples of images created by our approach in Fig. \ref{fig:gen_result}. The third column shows intermediate images that share the same pose as the first column (V), seek to incorporate the style of the second column (I), and lose the color information.

\floatsetup[figure]{style=plain,subcapbesideposition=center}
\begin{figure}[t]
  \centering
 \subfloat[][]{\includegraphics[width=.12\linewidth]{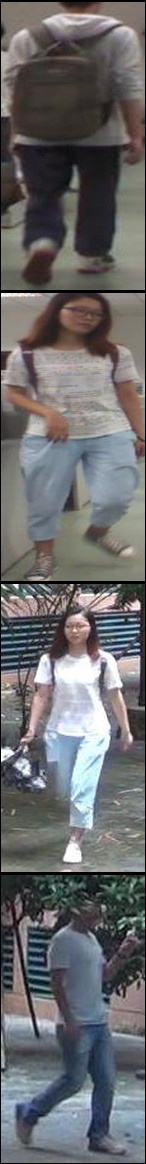}}
 \subfloat[][]{\includegraphics[width=.12\linewidth]{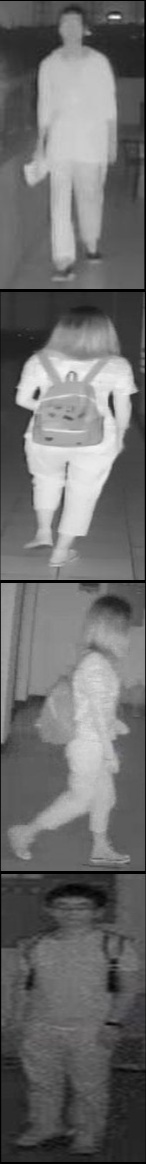}}  
 \subfloat[][]{\includegraphics[width=.12\linewidth]{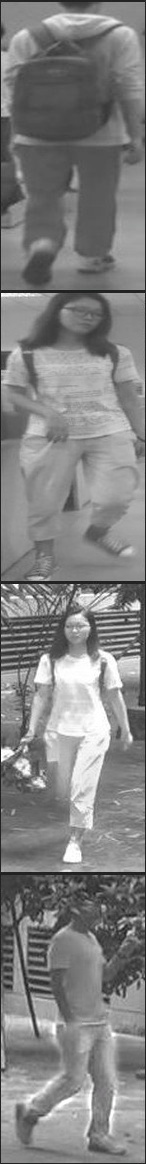}}  
 \subfloat[][]{\includegraphics[width=.12\linewidth]{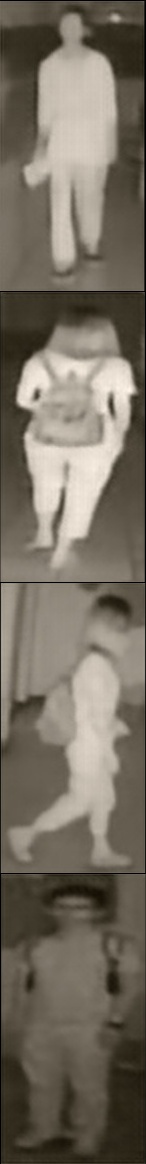}}  
 \subfloat[][]{\includegraphics[width=.12\linewidth]{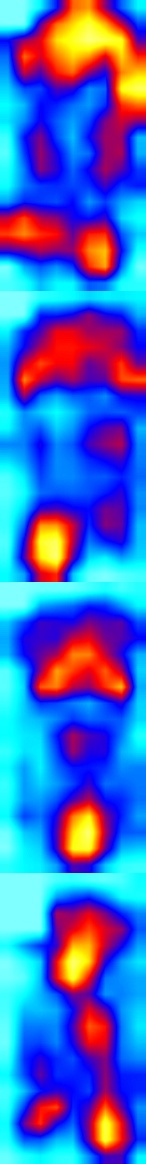}}  
 \subfloat[][]{\includegraphics[width=.12\linewidth]{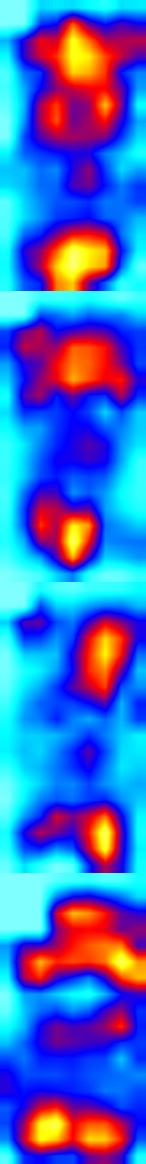}}  
 \subfloat[][]{\includegraphics[width=.12\linewidth]{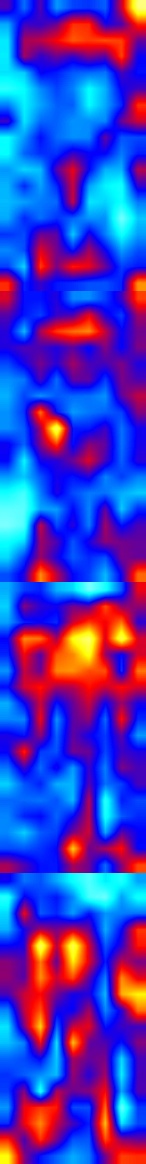}}  
 \subfloat[][]{\includegraphics[width=.12\linewidth]{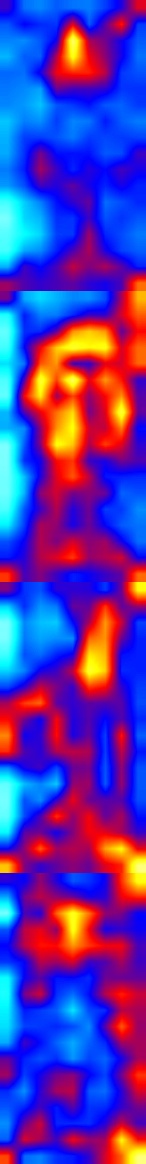}}  
  \caption{Examples of images generated with our \AG for SYSU-MM01 dataset. Columns (a) and (b) are V and I images. The intermediate and reconstructed I images are shown in (c) and (d), respectively. The GradCAM\cite{gradCAM} of our ID-Modality vs binary discriminator is shown in columns (e) to (h). Columns (e) and (f) for \AG discriminator for V and I images, respectively, and (g),(h) for Binary discriminator.}  
  \label{fig:gen_result}
\end{figure}  

\subsection{Ablation Studies} \label{sec:ablation}

A thorough ablation study is conducted to assess the impact of each module on \AG performance through various experiments on the SYSU-MM01 dataset. Initially, we compare our intermediate domain approach with other methods that use fixed and generated intermediates. Subsequently, to validate the roles of loss objectives, the generative module, and the ID-modality discriminator, we experiment with different settings and combinations with other methods.

\subsubsection{Effect of Intermediate Image Generation}
Table \ref{tab:gray-modal} shows the performance of cross-modality ReID with and without using an intermediate module. With the adaptive generated images (I-V-Z training), we achieve 56.35\%, while using the baseline model (I-V), we get 41.57\% mAP. Additionally, the adaptive version shows comparable performance in both V and I testing settings, making it an effective alternative for a robust person ReID. Table \ref{table:cross_SYSU_Result} provides different options for intermediate images and it is shown that, in comparison to other approaches using the intermediate domain, \AG has better performance.  

\setlength{\tabcolsep}{3.7pt}
\begin{table}[!t]
\centering
\caption{The impact on the accuracy of the proposed intermediate domain in SYSU-MM01 under the multi-shot setting.}
\label{tab:gray-modal}
\begin{tabular}{|c|c|c||r|r|}
\hline
 \multirow{2}{*}{\textbf{Training}} & \multicolumn{2}{c||}{\textbf{Testing}}                 & \multicolumn{1}{c|}{\multirow{2}{*}{\textbf{R1 (\%)}}} & \multirow{2}{*}{\textbf{mAP(\%)}} \\ \cline{2-3}
& \textbf{Query} & \textbf{Gallery} &  & \\ \hline \hline
\multirow{3}{*}{I-V}               & V              & V               & 97.40                                                  & 91.82                             \\ 
& I              & I               & 95.96                                                  & 80.49                             \\ 
                                       & I              & V               & 58.69                                                  & 41.57                             \\ 
\hline 
 \multirow{3}{*}{I-V-Z}             & V              & V               & 97.68                                                  & 90.19                             \\ 
                                     & I              & I               & 95.34                                                  & 81.62                             \\ 

                                    & I              & V               & 73.17                                                  & 56.35                             \\  \hline
\end{tabular}
\vspace{-.2cm}
\end{table}

\setlength{\tabcolsep}{3.4pt}
\begin{table}[!h]
\centering
\caption{Accuracy using different intermediate images on SYSU-MM01 using the \textbf{S}ingle and \textbf{M}ulti-shot gallery. "*" indicates that our model is trained on V images while tested with I. In "\$," we use I samples to train and test the model. "+" means that we only reported the MUN\cite{MUN} model performance of using auxiliary bringing part. 
}
\label{table:cross_SYSU_Result}
\resizebox{\linewidth}{!}{
\begin{tabular}{|l||c|c|c|c|}
\hline
\multirow{2}{*}{\textbf{Model}} & \multicolumn{2}{c|}{\textbf{R1(\%)}}                 & \multicolumn{2}{c|}{\textbf{mAP(\%)}}                 \\ \cline{2-5}
                       &  \textbf{S}     & \textbf{M}     & \textbf{S}    & \textbf{M}\\ \hline \hline
Pre-trained (test: I$\rightarrow$V)           & 2.26           & 3.31           & 3.61           & 1.71                                  \\ 
Lower (train: V, test: I$\rightarrow$V ) $^*$       & 7.41           & 9.86           & 8.83           & 6.83                                  \\ 
Upper (train: I, test: I$\rightarrow$I ) $^\$$   & 90.34          & 94.90          & 87.45          & 80.38                                 \\ \hline 

Baseline   (test: I$\rightarrow$V)            & 49.41          & 58.69          & 41.55          & 38.17                                 \\ \hline
Grayscale \cite{HAT} (test: I$\rightarrow$V)                 & 61.45          & 65.16          & 59.46          & 46.94                                 \\ 
RandG  \cite{randLUPI} (test: I$\rightarrow$V)              & 64.50          & 69.01          & 61.05          & 50.84                                \\ 
MUN$^+$ \cite{MUN} (test: I$\rightarrow$V)                 & 71.66          & -          & 68.35          & -                                 \\
\AG    (test: I$\rightarrow$V)          & \textbf{72.23} & \textbf{78.19} & \textbf{70.58} & \textbf{64.13}                       \\ \hline
\end{tabular}
}
\vspace{-.2cm}
\end{table}

\subsubsection{Effect of Losses}
In the following, we discuss the effect of each loss in \AG.

\noindent \textbf{Effect of Intermediate Dual Triplet Loss:}
Table \ref{table:losses} shows that when used in conjunction with the dual triplet loss, this loss provides supervision to improve performance at test time.
Although the ordinary triplet loss improves performance when it is integrated with $\mathcal{L}_{\textit{cf}}$, it ignores the privileged information provided by the intermediate domain. To show that fact, we compared the effect of our $\mathcal{L}_{\textit{dual}}$ to $\mathcal{L}_{\textit{tri}}$ by using each of them alongside other losses in the training process. As can be seen, using such regularization increases mAP and rank-1 by 3\% and 7\%, respectively, over ordinary triplet loss. 

\noindent \textbf{Effect of Color-Free Loss:}
Table \ref{table:losses} illustrates the impact of $\mathcal{L}_\text{cf}$. 
The results indicate that incorporating this regularizing loss leads to an increase of 3\% in mAP and 5\% in rank-1 accuracy when it is used without using $\mathcal{L}_{\text{tri}}$ or $\mathcal{L}_{\text{dual}}$. When the model is trained with $\mathcal{L}_{\text{cf}}$ and $\mathcal{L}_{\text{tri}}$ the performance significantly improves (19\% in mAP and 12\% in rank-1) and alongside with $\mathcal{L}_{\text{dual}}$ mAP and rank-1 improved over 6\% and 7\% respectively. As the intermediate images are derived from visible images with color information removed and infrared style information incorporated, $\mathcal{L}_\text{cf}$ plays a crucial role. This ensures that the extracted features from the visible modality are color-independent, aligning them more closely with the attributes of infrared.

\noindent \textbf{Effect of Adversarial Loss:}
As a result of the adversarial loss, the generator must adapt the reconstruction process, enabling the stylistic transformation of images from V to I.
Without using this objective, the generated images are similar to the gray-scale channel images and do help the model adjust for differences between I and V images. As you can see in the 4th row of Table. \ref{table:losses}, the performance is close to the model trained based on grey intermediate images in \cite{HAT}.


\noindent \textbf{Effect of Reconstruction Loss:}
To improve the quality of the reconstructed I images, supervised reconstruction loss is applied to the encoder and decoder parts. In the absence of this loss, the generated images lack sufficient important information to support the feature module and make the performance drop dramatically. As the color-free loss pushes the V features to be the same as the features of such not meaningful intermediate images, the mAP accuracy of matching between V and I decreases (from 41.55\% to 5.51\%) as seen in the 5th row of Table \ref{table:losses}.

\setlength{\tabcolsep}{3.3pt}
\begin{table}[!t]
\small
\centering

\caption{Impact on the accuracy of using different losses in our \AG on SYSU-MM01 in single-shot.}
\label{table:losses}

\resizebox{\linewidth}{!}{
\begin{tabular}{|c|c|c|c|c|c||r|r|}
\hline
\multicolumn{6}{|c||}{\textbf{Losses}} & \multirow{2}{*}{\textbf{R1 (\%)}} & \multirow{2}{*}{\textbf{mAP (\%)}} \\ \cline{1-6}
$\mathcal{L}_{\text{id}}$ & $\mathcal{L}_{\text{tri}}$ & $\mathcal{L}_{\text{dual}}$ & $\mathcal{L}_{\text{cf}}$ & $\mathcal{L}_{\text{adv}}$ & $\mathcal{L}_{\text{rec}}$  &                     &  \\ \hline   \hline                   
\cmark  & \cmark   &                 &               &              &              & 49.41                                   & 41.55                                    \\
\cmark         &                  & \cmark           &              & \cmark          & \cmark          & 60.64                                   & 56.15                                    \\
\cmark         &                   &                 &          & \cmark          & \cmark          & 40.16                                   & 35.89                                    \\

\cmark         &                   &                 & \cmark         & \cmark          & \cmark          & 45.56                                   & 41.33                                    \\
\cmark         &                  & \cmark           & \cmark         &                & \cmark          & 60.67                                   & 58.35                                    \\
\cmark         &                  & \cmark           & \cmark         & \cmark          &                & 9.32                                    & 5.51                                     \\
\cmark         &  \cmark                       &            &   \cmark       & \cmark          & \cmark          & 62.76                                   & 60.02       \\ 
\cmark         &  \cmark                       &     \cmark       &   \cmark       & \cmark          & \cmark          & 64.80                                   & 61.92       \\ 

\cmark         &                          & \cmark           & \cmark         & \cmark          & \cmark          & 67.29                                   & 63.86       \\ \hline                            
\end{tabular}
}
\vspace{-.2cm}
\end{table}

\subsubsection{Generation Model} \label{sec:exp_gen}

To find the highest ReID performance with respect to the quality of the generated intermediate images, we tested different generative models by adapting V to I modality or vice versa. Since V images have more information to discriminate between individuals in comparison to I \cite{SIMcm}, the transferred V images play a better role in providing privileged information in the training process. In Table \ref{tab:gen}, the model in \cite{van2017neural} is deeper (more layers) and has a better capacity to stylize the V images from the I feature vectors. Also, we tested a large state-of-the-art model \cite{choi2020starganv2}, which is used for multiple domain adaptation as our generator.

\begin{table}[!h]
\centering
\setlength\tabcolsep{2.5pt}
\caption{Accuracy with different options for the generation of intermediate images.}
\label{tab:gen}
\begin{tabular}{|l||ll|ll|l|}
\hline
\multirow{2}{*}{\textbf{Model}} & \multicolumn{2}{c|}{\textbf{V $\rightarrow$ I}}                              & \multicolumn{2}{c|}{\textbf{I $\rightarrow$ V}}                               & \multicolumn{1}{c|}{\multirow{2}{*}{\textbf{\# params}}} \\ \cline{2-5}
      & \multicolumn{1}{c|}{\textbf{R1}} & \multicolumn{1}{c|}{\textbf{mAP}} & \multicolumn{1}{c|}{\textbf{R1}} & \multicolumn{1}{c|}{\textbf{mAP}} & \multicolumn{1}{c|}{}                                    \\ \hline \hline
VQ-VAE2                                                & \multicolumn{1}{l|}{65.75}           & 61.96                                 & \multicolumn{1}{l|}{61.81}           & 58.34                                 &  20M                                                        \\
deep VQ-VAE2                                           & \multicolumn{1}{l|}{\textbf{67.29}}  & \textbf{63.86}                        & \multicolumn{1}{l|}{\textbf{63.45}}  & \textbf{60.09}                        &   28M                                                       \\
StarGAN                                               & \multicolumn{1}{l|}{66.21}           & 62.70                                 & \multicolumn{1}{l|}{62.41}           & 59.12                                 &     50M                                                    \\ \hline
\end{tabular}

\end{table}


\subsubsection{ID-modality Discriminator vs General Discriminator}
The ID-modality discriminator aims to discriminate modalities based on ID-aware features for each modality to help the generative module transform such features from V and I to intermediate. To show the effectiveness of such a module on ReID performance, we compared its results with those of a general binary discriminator that uses whole features.   
Table \ref{tab:disc} demonstrates a 4\% improvement in rank-1 accuracy and a 5\% enhancement in mAP measurement by introducing our ID-modality discriminator. The discriminator effectively identifies modality-specific, ID-related information, enabling the generative model to transform visible images into infrared while preserving the individuals' discriminative attributes. Consequently, the transformed images provide intermediate privileged information to the feature backbone, leading to a more robust reduction of modality discrepancies in the feature space. Additionally, Fig. \ref{fig:gen_result} highlights the key regions of the input images that the discriminator focuses on when making its decisions, demonstrating the relevance of the identified features.

\begin{table}[!h]
\centering
\setlength\tabcolsep{2.5pt}
\caption{Accuracy with different options for the discriminator and generation of intermediate images.}
\label{tab:disc}
\begin{tabular}{|l||ll|ll|}
\hline
\multicolumn{1}{|c||}{\multirow{2}{*}{\textbf{Model}}} & \multicolumn{2}{c|}{\textbf{V $\rightarrow$ I}}                              & \multicolumn{2}{c|}{\textbf{I $\rightarrow$ V}}                               \\ \cline{2-5}
                                & \multicolumn{1}{c|}{\textbf{R1}} & \multicolumn{1}{c|}{\textbf{mAP}} & \multicolumn{1}{c|}{\textbf{R1}} & \multicolumn{1}{c|}{\textbf{mAP}}                                     \\ \hline \hline
Binary Discriminator                                                & \multicolumn{1}{l|}{63.23}           & 58.84                                 & \multicolumn{1}{l|}{60.18}           & 57.32                                                                       \\
Our Discriminator                                           & \multicolumn{1}{l|}{\textbf{67.29}}  & \textbf{63.86}                        & \multicolumn{1}{l|}{\textbf{63.45}}  & \textbf{60.09}                                                                   \\ \hline
\end{tabular}

\end{table}

\subsection{Inference Model Efficiency } \label{sec:exp_gen}
Our \AG proposed method creates a privileged intermediate domain only in training to improve the ReID performance without adding extra computation overhead in the inference time. In Table \ref{tab:infrence_complex}, we compared our model's size and time complexity in the inference time with the baseline\cite{all-survey}, MPANet\cite{wu2021Nuances}, MMN\cite{zhang2021towards}, DTRM \cite{DDAG2}, RPIG \cite{randLUPI} and thGAN \cite{kniaz2018thermalgan}.

\begin{table}[ht]
\centering
\setlength\tabcolsep{2.5pt}
\caption{Number parameters and floating-point operations at inference time for AGPI$^2$ and state-of-the-art models. }
\label{tab:infrence_complex}
\begin{tabular}{|l|c|c|}
\hline
\textbf{Model} & \textbf{\# of Para. (M)} & \textbf{Flops (G)}  \\ \hline \hline
baseline \cite{all-survey}  &  \textbf{24.8}   & \textbf{5.2}    \\
\AG  &  \textbf{24.8}  & \textbf{5.2}  \\ \hline
MPANet \cite{wu2021Nuances}  &  30.1  & 6.7   \\
MMN \cite{zhang2021towards}  &  26.6  & 5.8    \\
DTRM \cite{DDAG2}  &  26.6  & 5.4  \\ 
RPIG \cite{randLUPI}  &  24.9  & \textbf{5.2}  \\ 
thGAN \cite{kniaz2018thermalgan}  &  35  & 8.9  \\ \hline
\end{tabular}

\end{table}

\subsection{Data Augmentation}
Random erasing augmentation \cite{zhong2020random} was assessed with different settings for the feature embedding and the generation modules. Table \ref{tab:aug} shows the accuracy of our proposed approach using different scenarios \cite{zhong2020random}. Despite each augmentation configuration improving performance, the best results are achieved when the augmentation probability is gradually increased during training.

\begin{table}[ht]
\small
\centering
\setlength\tabcolsep{2.5pt}
\caption{Matching accuracy with random erasing augmentation. $s$ is the size of the rectangle and $p$ is the probability of erasing. Note: "$a \rightarrow b$" means that training started from $a$ and finished with $b$.}
\label{tab:aug}
\begin{tabular}{|l||c|c||c|c|}
\hline
\textbf{Configurations} & \textbf{$p$}            & \textbf{$s$}         & \textbf{R1 (\%)} & \textbf{mAP(\%)} \\ \hline \hline
\AG                & 0                     & -                     & 67.29            & 63.86            \\
\AG + Aug1         & 0.40                   & 0.20                   & 69.81            & 65.19            \\
\AG + Aug2         & 0.50                   & 0.25                  & 70.93            & 66.15            \\
\AG + Aug3         & 0.50                   & 0.30                   & 69.39            & 64.95            \\
\AG + Aug4         & 0.30 $\rightarrow$ 0.50 & 0.20 $\rightarrow$ 0.30 & \textbf{72.23}   & \textbf{70.58}   \\ \hline
\end{tabular}

\end{table}

\begin{figure}[!ht]
\centering
\begin{minipage}{.5\linewidth}
\centering
\subfloat[Training set(baseline)]{\includegraphics[width=\linewidth]{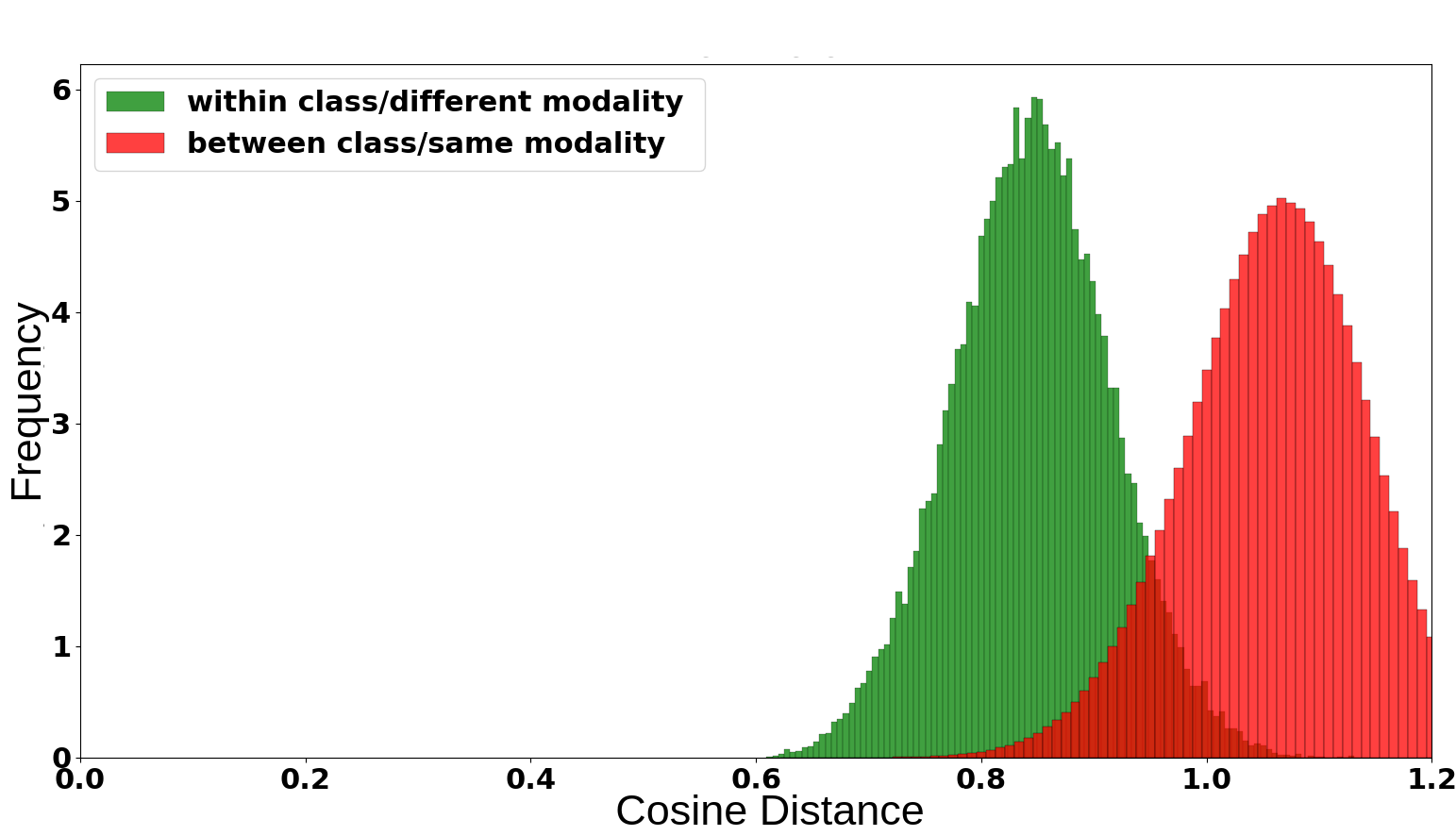}}
\end{minipage}%
\begin{minipage}{.5\linewidth}
\centering
\subfloat[Test set(baseline)]{\includegraphics[width=1\linewidth]{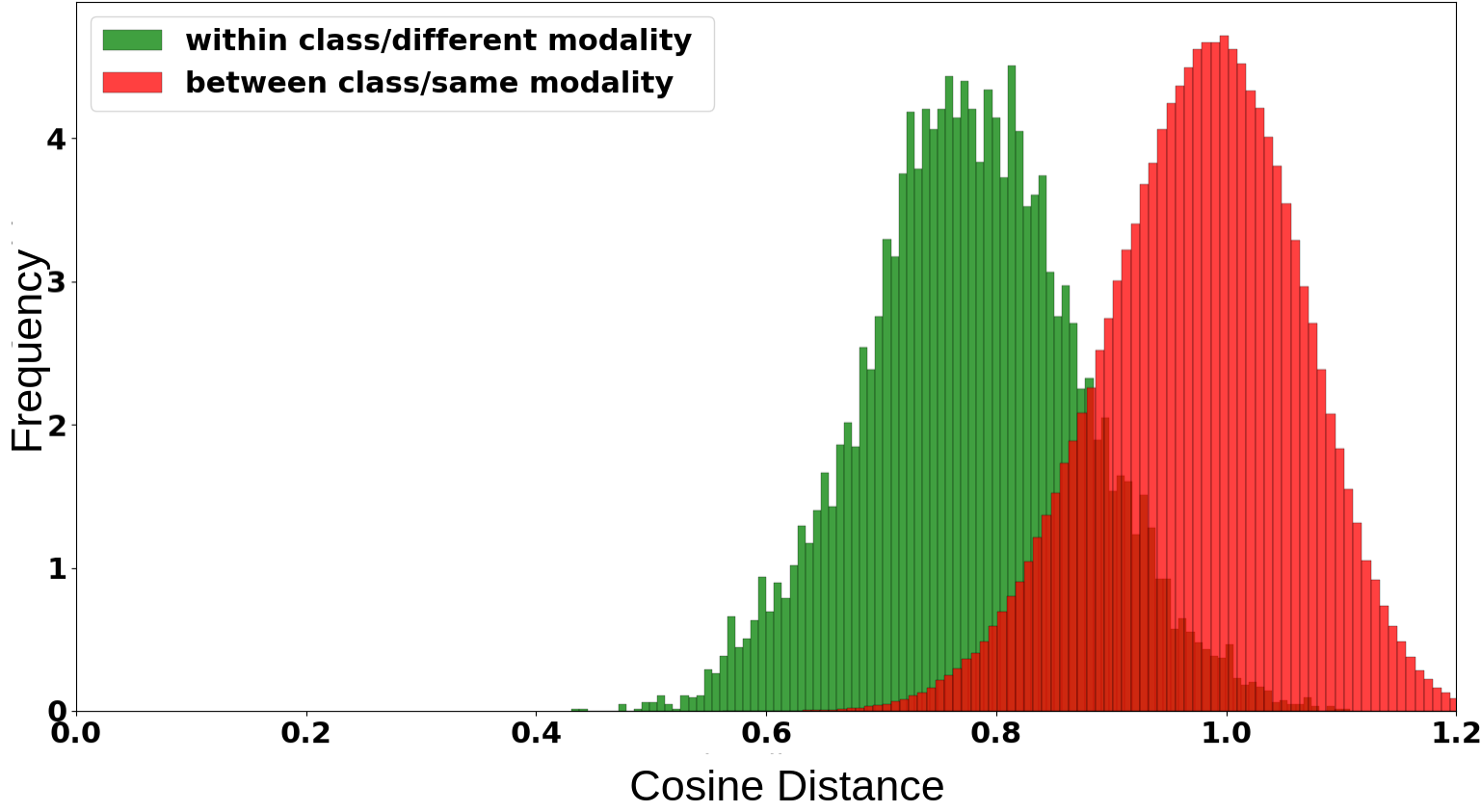}}
\end{minipage}\par\medskip 
\begin{minipage}{.5\linewidth}
\centering
\subfloat[Training set(\AG).]{\includegraphics[width=1\linewidth]{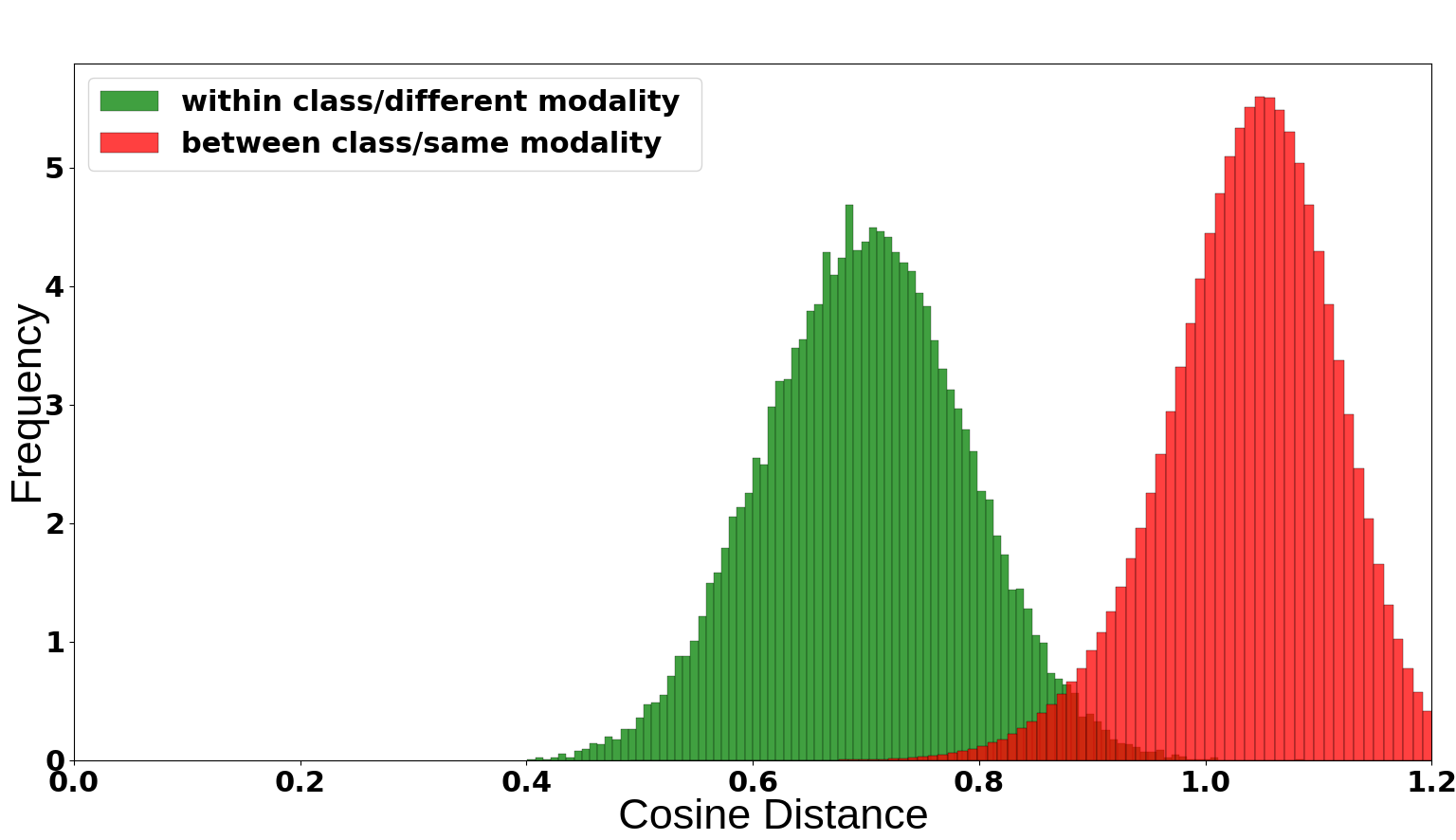}}
\end{minipage}%
\begin{minipage}{.5\linewidth}
\centering
\subfloat[Test set(\AG)]{\includegraphics[width=1\linewidth]{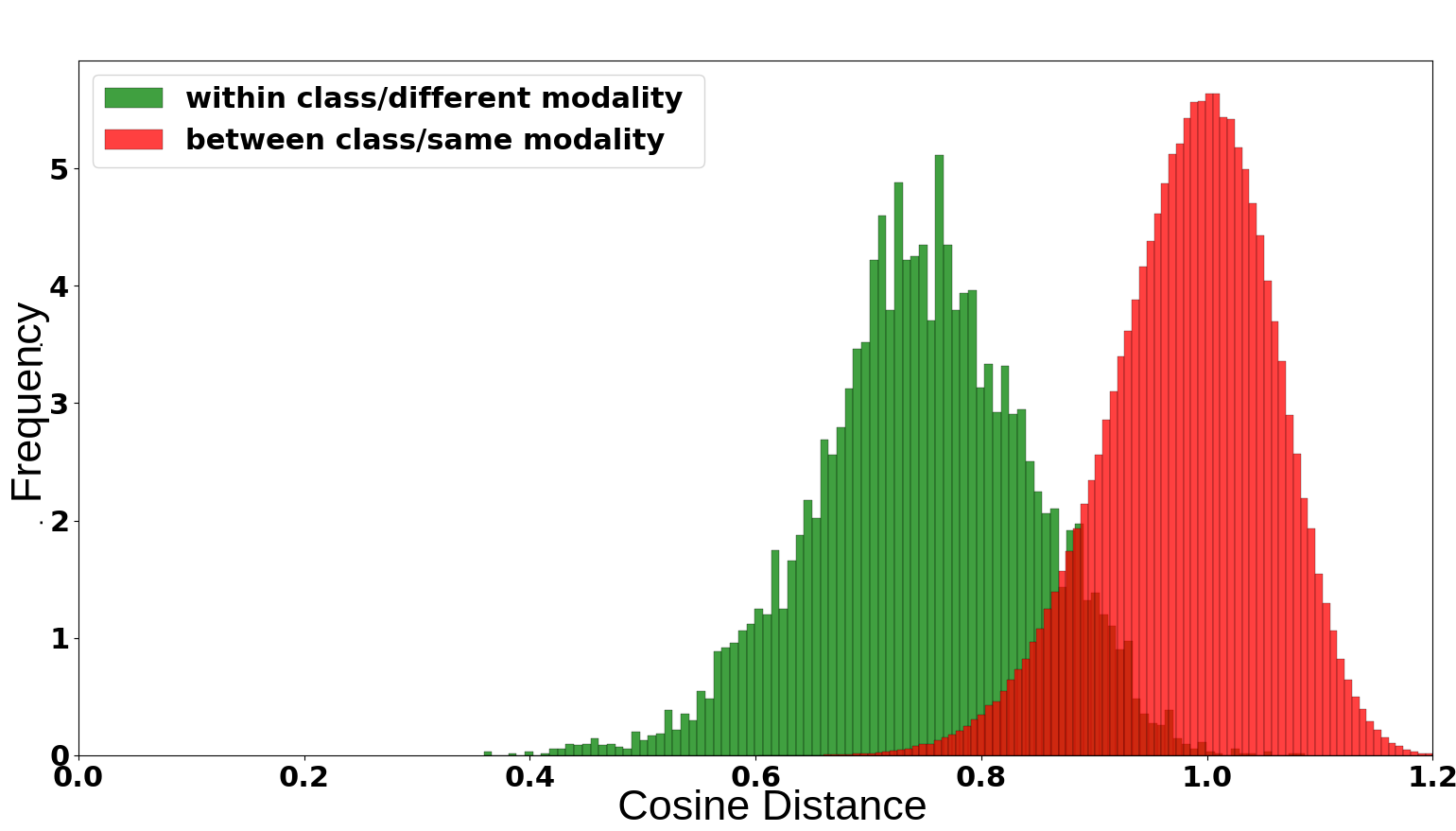}}
\end{minipage}\par\medskip
\caption{Distribution of between- and within-class distances on the  SYSU-MM01 dataset for (a,c) train and (b,d) test sets with the baseline model\cite{all-survey} and our \AG model.}
\label{fig:q1}
\end{figure}

\subsection{Qualitative Evaluation}

A visual representation of the cosine distance distributions for positive and negative cross-modality matching pairs is provided in Fig. \ref{fig:q1} for both training and testing sets. The proposed \AG approach is compared to a strong baseline (AGW model in \cite{all-survey}). 
After training, such relationships should be reversed so that the two strategies work effectively to decouple I from V. Despite this, \AG is more efficient than the baseline on the test set when combined with the adaptively generated images. By leveraging privileged intermediate information during training, \AG improves the model's ability to discriminate between V and I images at test time.
\begin{figure*}
\centering
\sidesubfloat[]{\includegraphics[width=0.9\linewidth]{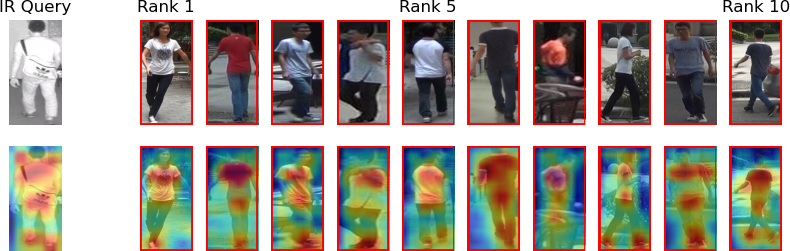}} \\
\sidesubfloat[]{\includegraphics[width=0.9\linewidth]{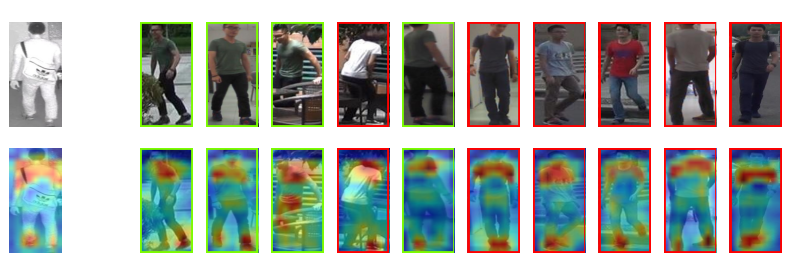}}
\caption{Rank-10 visualizations of the (a) baseline (b) \AG strategies on SYSU-MM01 data, along with their corresponding similarity CAM \cite{stylianou2019visualizing}. Green or red boxes respectfully surround correct and incorrect matches. There are only four visible cameras, and only one image from one of them is in the gallery. }
\label{fig:vizualizations_rank_10}
\end{figure*}

Fig. \ref{fig:vizualizations_rank_10} shows the rank-10 for a given class of the SYSU-MM01 dataset given by the AGW \cite{all-survey} (the baseline) and our proposed strategy, along with the corresponding similarity CAMs \cite{stylianou2019visualizing}. The query CAM is produced from the similarity of the best match. Compared to the baseline method, most discriminative regions and attention between them are refined and more focused. CAMs seem to focus on shoulders, hips, and feet for the baseline method, while the hip's attention almost fully disappears for our model. Interestingly, this behavior is consistent from one sample to another. Perhaps the hip focus at the baseline comes from the important color changes between the lower and upper bodies, which are probably discriminant for the visible part of the model. However, regarding the IR modality, this color marker is much less reliable and probably appears as a source of confusion for cross-modal ReID. This phenomenon tends to diminish thanks to the adaptive and informative intermediate privileged domain that is created through the training.    

To demonstrate the effectiveness of our proposed ID-Modality discriminator, inspired by \cite{temp1, temp2}, we compared its class activation map to that of a general binary discriminator, as shown in Fig. \ref{fig:gen_result}. The results indicate that the ID-Modality discriminator is more focused on the foreground, whereas the binary discriminator primarily detects the modality based on the background.

\section{Conclusion}

In this paper, a new framework is introduced for adapting privileged intermediate information called \AG. This method generates a virtual intermediate domain that bridges the gap between V and I images during training, thereby reducing the modality discrepancy problem. To create privileged intermediate information, an adversarial approach is also introduced for the generator, feature embedding, and ID-modality discriminator modules in the training process. AG outperforms state-of-the-art methods for V-I person ReID on the SYSU-MM01 and RegDB datasets while incurring no computational overhead during inference. 

The \AG training approach can be integrated into other SOTA methods for ReID to improve performance and generality. 
Compared to other intermediate generation methods, we show \AG approach can provide more balance between V and I information for the model to diminish the modality-specific attributes in the representation feature space. 
In future research endeavours, we aim to extend the application of the \AG approach to multi-step and gradual intermediate privileged learning, particularly to address more substantial domain shift challenges. This will necessitate the \AG model's ability to quantify the incorporation of I domain information during the transformation of V images into the intermediate space, facilitating a gradual adaptation process.




\bibliographystyle{IEEEtran}
\bibliography{egbib}


\newpage

 
\begin{IEEEbiography}
[{\includegraphics[width=1in,height=1.2in,clip,keepaspectratio]{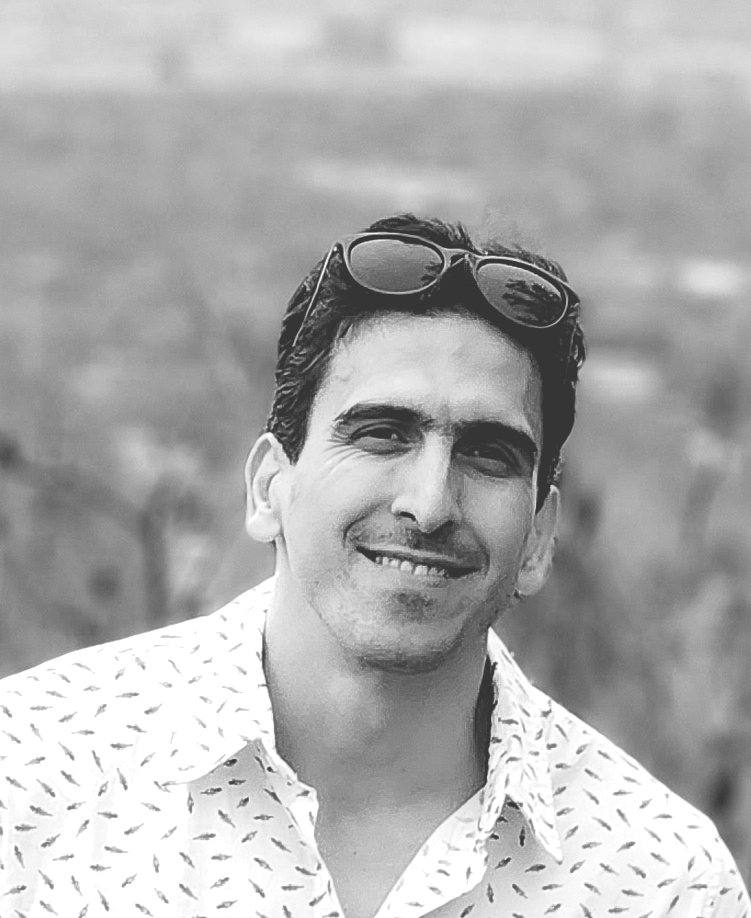}}] 
{Mahdi Alehdaghi} received an M.S. degree in artificial intelligence in 2015 from the Ferdowsi University of Mashhad, Mashhad-Iran. He is currently pursuing a Ph.D. degree in computer engineering from the École de Technologie Supérieure (ÉTS), Montréal, Quebec. His current research interests include computer vision, multimodal representations, and gradual and multi-step domain generalization.
\end{IEEEbiography}

\begin{IEEEbiography}
[{\includegraphics[width=1in,height=1.2in,clip,keepaspectratio]{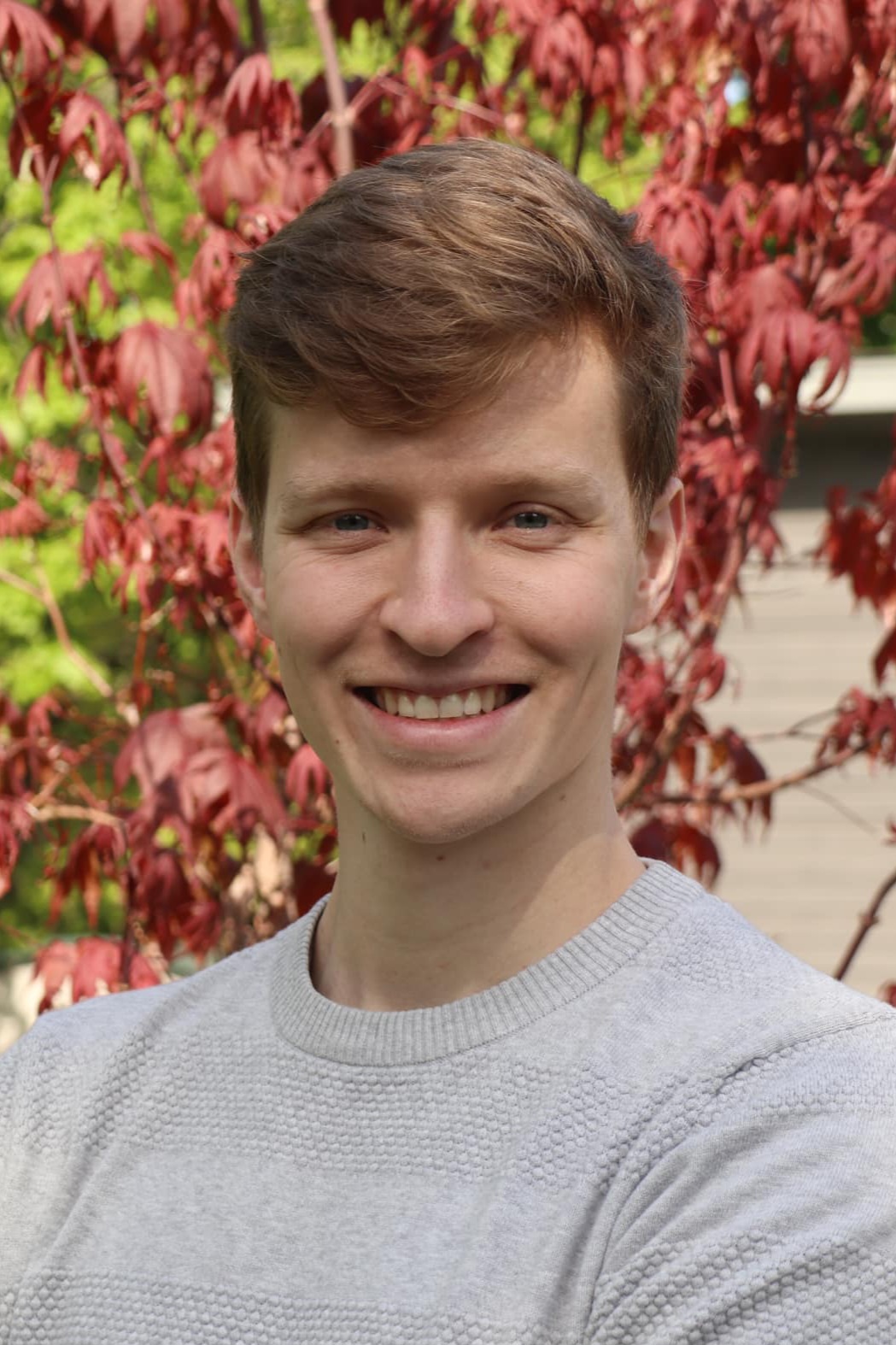}}] 
{Arthur Josi} received an M.Sc degree in artificial intelligence in 2023 from the 'École de technologie superieure (ÉTS)' Montréal, Québec. He is currently pursuing a PhD in computer engineering from the ÉTS. His current research interests are in computer vision, multimodal fusion, 3D facial animation, and deep generative models. 
\end{IEEEbiography}

\begin{IEEEbiography}
[{\includegraphics[width=1in,height=1.2in,clip,keepaspectratio]{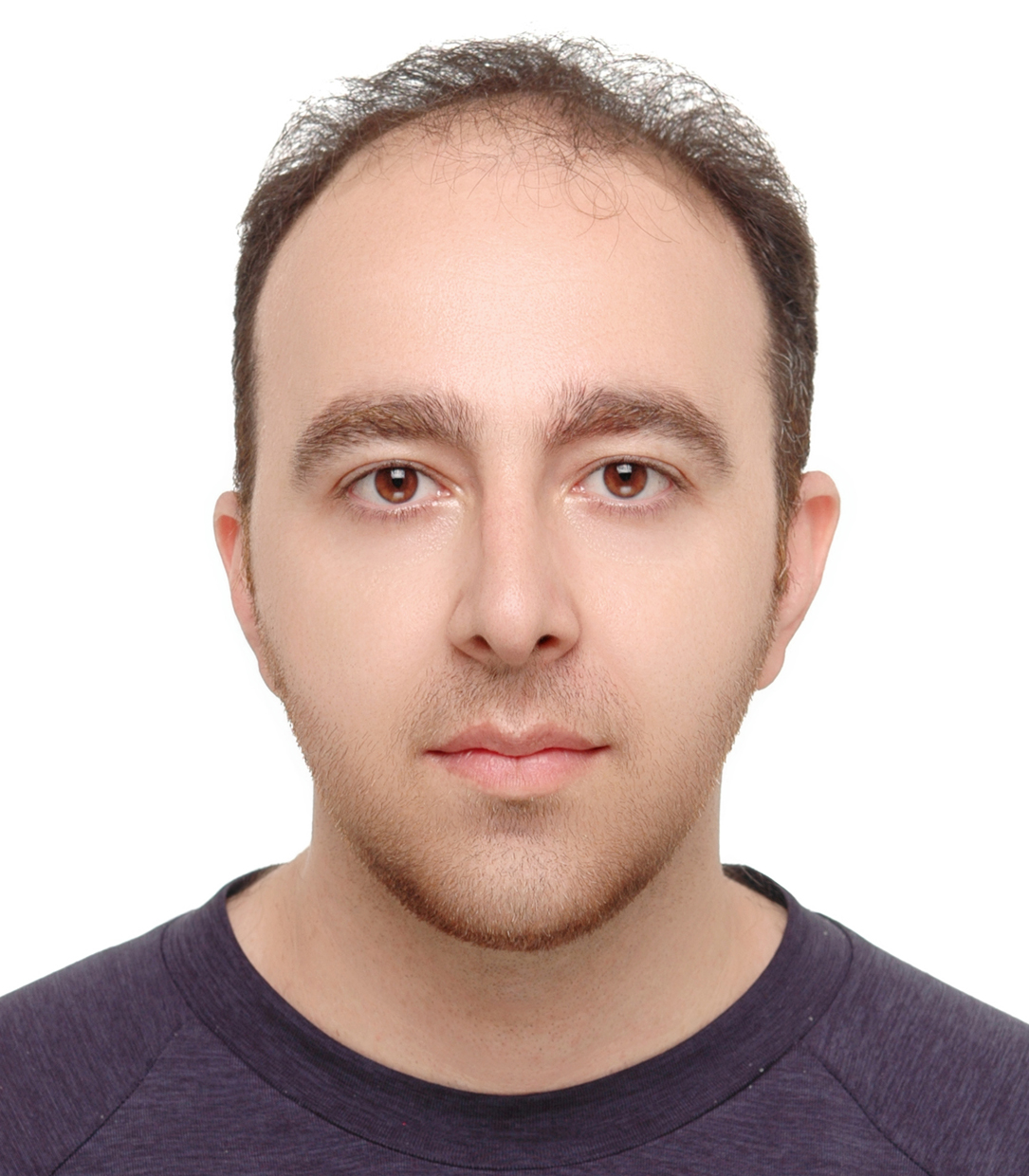}}] 
{Pourya Shamsolmoali} (Senior Member, IEEE) received the Ph.D. degree in computer science from Shanghai Jiao Tong University. He has been a Visiting Researcher with CMCC-Italy, INRIA-France, and ETS-Canada. He is currently a lecturer at the University of York with expertise in image processing, computer vision, and machine learning.
\end{IEEEbiography}
\begin{IEEEbiography}
[{\includegraphics[width=1in,height=1.2in,clip,keepaspectratio]{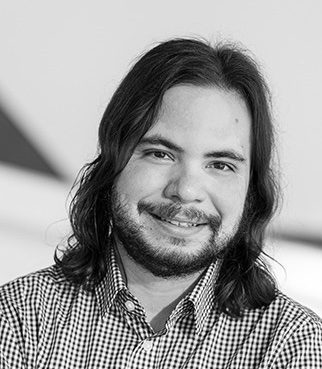}}] 
{Rafael M. O. Cruz} (Member, IEEE) received a Ph.D. in engineering in 2016 from the École de Technologie Supérieure (ÉTS), Montréal-Quebec, where he is currently an Associate Professor. His current research interests include ensemble learning, data complexity, dynamic ensemble models, meta-learning, and biometrics.
\end{IEEEbiography}
\begin{IEEEbiography}
[{\includegraphics[width=1in,height=1.2in,clip,keepaspectratio]{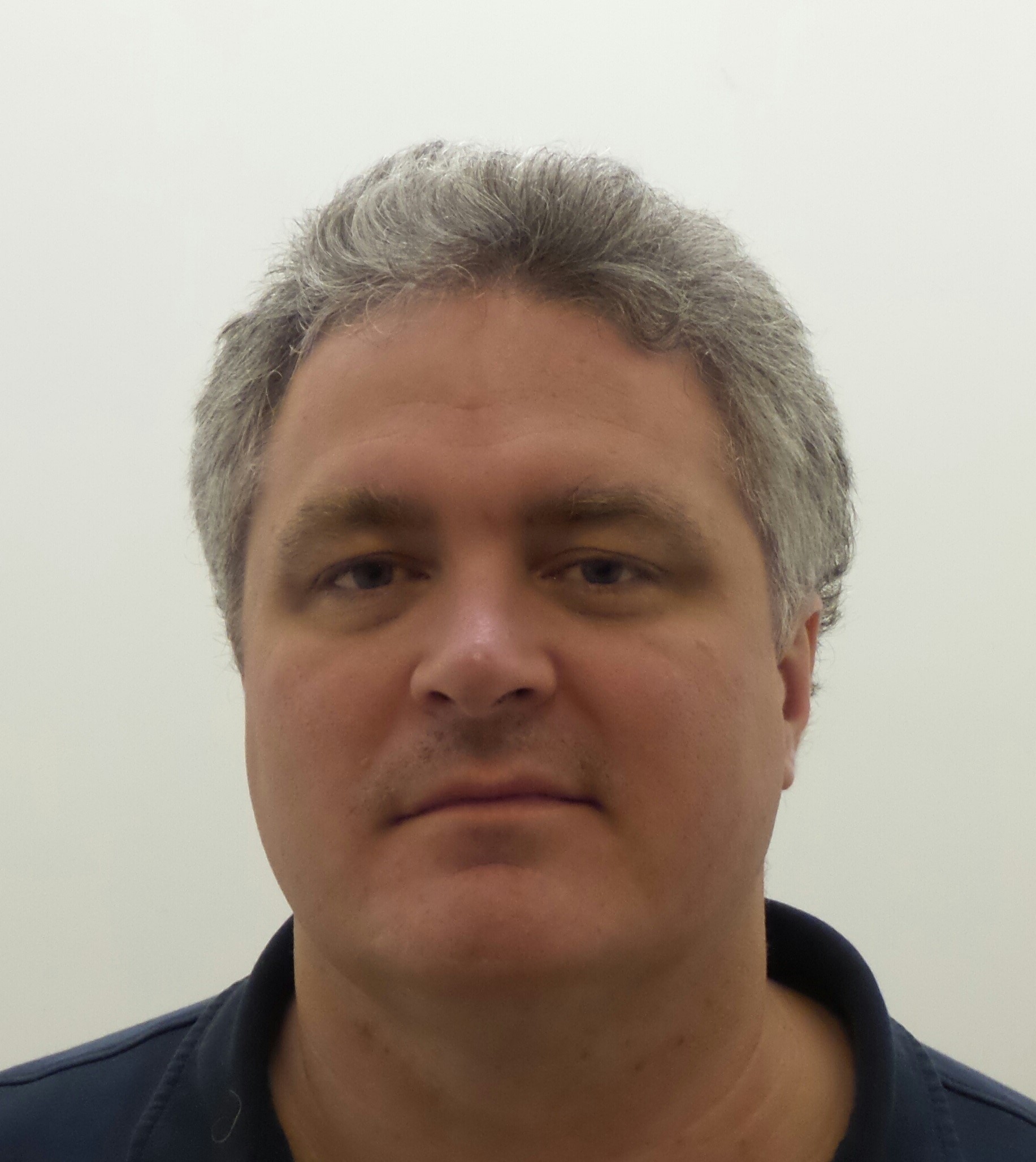}}]{Eric Granger } (Member, IEEE) received a Ph.D. degree in electrical engineering from École Polytechnique de Montréal in 2001. He was a Defense Scientist with DRDC Ottawa, from 1999 to 2001, and in Research and Development with Mitel Networks from 2001 to 2004. He joined the Department of Systems Engineering, ETS Montréal, Canada, in 2004, where he is currently a Full Professor and the Director of the LIVIA. He is the FRQS Co-Chair in AI and Health, and the ETS Industrial Research Co-Chair on embedded neural networks for intelligent connected buildings (Distech Controls Inc.). His research interests include pattern recognition, machine learning, information fusion, and computer vision, with applications in affective computing, biometrics, face recognition, medical image analysis, and video surveillance.
\end{IEEEbiography}
\vfill
\input{suppl}

\end{document}

%% file: suppl.tex
\clearpage
\newtheorem{theorem}{Theorem}
\newtheorem{proof}{Proof}

\newtheorem{prop}{Proposition}


\appendices

\numberwithin{equation}{section}
\numberwithin{figure}{section}
\numberwithin{table}{section}
\renewcommand\thesubsectiondis{\arabic{subsection}.}

\setcounter{page}{1}

\section{Proofs} 
\subsection{Maximizing $\MI(Z;Y)$} \label{sec:prof1}
\begin{prop}
\label{prop:cross}
Let $Z$ and $Y$ be random variables with domains $\mathcal{Z}$ and $\mathcal{Y}$, respectively. Minimizing the conditional cross-entropy loss of predicted label $\hat{Y}$, denoted by $\mathcal{H}(Y; \hat{Y}|Z)$, is equivalent to maximizing the $\MI(Z; Y)$
\end{prop}
\begin{proof}
    Let us define the MI as entropy,
    \begin{equation}
        \MI(Z;Y) = \underbrace{\mathcal{H}(Y)}_{\delta} - \underbrace{\mathcal{H}(Y|Z)}_{\xi}
    \end{equation}
    Since the domain $\mathcal{Y}$ does not change, the entropy of the identity $\delta$ term is a constant and can therefore be ignored. Maximizing $\MI(Z;Y)$ can only be achieved by minimizing the $\xi$ term. By expanding its relation to the cross-entropy \cite{boudiaf2020unifying}:
    \begin{equation}
        \label{eq:cross}
        \mathcal{H}(Y, \hat{Y}|Z) = \mathcal{H}(Y|Z) + \underbrace{\mathcal{D}_{\text{KL}}(Y||\hat{Y}|Z)}_{\geq 0} , 
    \end{equation}
    where:
    \begin{equation}
        \mathcal{H}(Y|Z) \leq \mathcal{H}(Y, \hat{Y}|Z).
    \end{equation}
    So, $\mathcal{H}(Y|Z)$ is upper-bounded by our cross-entropy loss, and minimizing such loss results in minimizing the $\xi$ term.

    Through the maximization ID-Aware part of Eq. \ref{eq:mi2}, training can naturally be decoupled in 2 steps. First, weights of the generative and feature backbone modules are fixed, and only the classifier parameters (i.e., the weight of the fully connected layer at the feature backbone) are minimized w.r.t. Eq. \ref{eq:cross}. Through this step, $\mathcal{D}_{\text{KL}}(Y||\hat{Y}|Z)$ is minimized by adjusting $\hat{Y}$ while the $\mathcal{H}(Y|Z)$ does not change.
    In the second step, the prototype module’s weights are minimized w.r.t. $\mathcal{H}(Y|Z)$, while the classifier parameters are fixed.
\end{proof}

\section{Additional Results}

\subsection{Hyperparameter Values} \label{sec:exp_hyper}
This subsection provides an analysis of the impact on ReID accuracy of the $\lambda_a$ and $\lambda_{cf}$ hyperparameters. $\lambda_a$ balances how much the model preserves the content of input images and the information of the style image, whereas $\lambda_{cf}$ pushes the similarity of the intermediate feature on V extracted features. Fig. \ref{fig:param_an} a) shows that accuracy improves when $\lambda_a$ increases until it reaches a value of 0.1. Then, the performance begins to decline. Similar results can be observed when $\lambda_{cf}$ varies from 1 to 10 (Fig. \ref{fig:param_an} b)). The performance of our proposed \AG strategies varies depending on these hyperparameter values, but most results exceed the state-of-the-art methods. 
\begin{figure}[ht]
\centering
\subfloat[Accuracy vs. $\lambda_{\text{adv}}$ value.]{\includegraphics[width=0.5\linewidth]{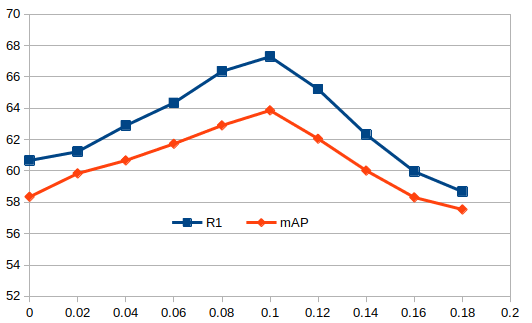}}
\centering
\subfloat[Accuracy vs. $\lambda_{cf}$ value.]{\includegraphics[width=0.5\linewidth]{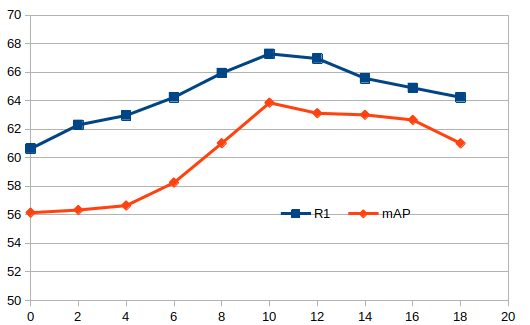}}
\caption{Impact on accuracy of $\lambda_{\text{adv}}$ and $\lambda_{cf}$ hyperparameters on the SYSU-MM01 dataset under the single-shot setting.}
\label{fig:param_an}
\end{figure}

\subsection{Comparison to CycleGAN}
 we conducted a new experiment to show the comparison between using intermediate and directly generated modalities in the training and inference phases separately. In Table. \ref{table:cycleGAN}, in the first two rows, the results for \AG are reported in which the generated images are used as an intermediate. The $\mathcal{L}_{\textit{dual}}$ has been applied in the training process. In comparison, in the two last rows, the generated images are used as other modality (Visible images are replaced with generated images in training) and the $\mathcal{L}_{\text{tri}}$ is applied between I and generated images. In the second and fourth rows, the generated images are used as the gallery instead of the visible images. \AG, which generates and uses images as an intermediate, performs better than CycleGAN\cite{CycleGAN2017} using generated images directly as another modality.  We can conclude that although the transformed images try to lose style information from the source and achieve them from the target domain, they are likely to have some content information from the target and style from the source, so the provided privileged information by them is not enough for the model to compensate the other modality in the training. 

\begin{table}[!t]
\small
\centering
\begin{tabular}{|l|c|c|c|c|c|c|}
\hline
\multicolumn{1}{|c|}{\multirow{2}{*}{\textbf{Method}}} & \multicolumn{2}{c|}{\textbf{Use}}                               & \multicolumn{2}{c|}{\textbf{Performance}}                           & \multicolumn{2}{c|}{\textbf{Complexity}} \\ \cline{2-7} 
\multicolumn{1}{|c|}{}                        & \textbf{Train} & \textbf{Test} & \textbf{R1(\%)} & \textbf{mAP(\%)} & \textbf{\# of Para.} & Flops \\ \hline
\multicolumn{7}{|c|}{Using generative images as intermediate in the training}   \\ \hline
\AG          &   \cmark    &  \xmark    & 72.23  & 70.58  & 24.8M           & 5.2G       \\
\AG                     &   \cmark    &   \cmark   & 65.76  & 62.23   & 52.8M            & 8.3G       \\ \hline
\multicolumn{7}{|c|}{Using generative images as other modality in the t
raining} \\ \hline
CycleGAN                &  \cmark    &  \xmark   & 42.05  & 41.17  & 24.8M          & 5.2G       \\
CycleGAN                &  \cmark &  \cmark & 26.91  & 28.75   & 42.5M            & 7.1G       \\ \hline
\end{tabular}
\captionsetup{width=\linewidth}
\caption{Impact of using generated images as intermediate versus using as visible in training or testing for SYSU dataset. \AG method uses them as intermediate in $\mathcal{L}_\textit{dual}$ and CycleGAN\cite{CycleGAN2017} uses as visible modality and applies the $\mathcal{L}_\textit{tri}$. }
\label{table:cycleGAN}
\end{table}

 \floatsetup[figure]{style=plain,subcapbesideposition=center}
\begin{figure}[!h]
  \centering
 \subfloat[][]{\includegraphics[width=.15\linewidth,height=.9\linewidth]{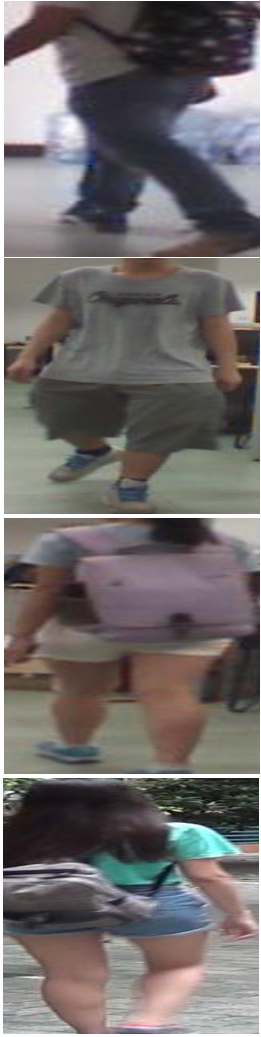}}
 \subfloat[][]{\includegraphics[width=.15\linewidth,height=.9\linewidth]{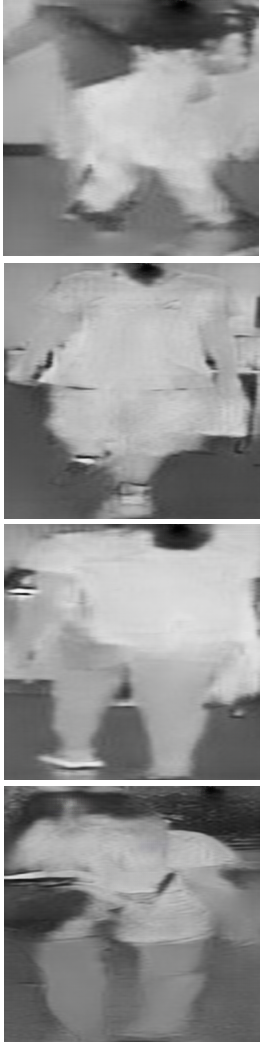}}  
 \subfloat[][]{\includegraphics[width=.15\linewidth,height=.9\linewidth]{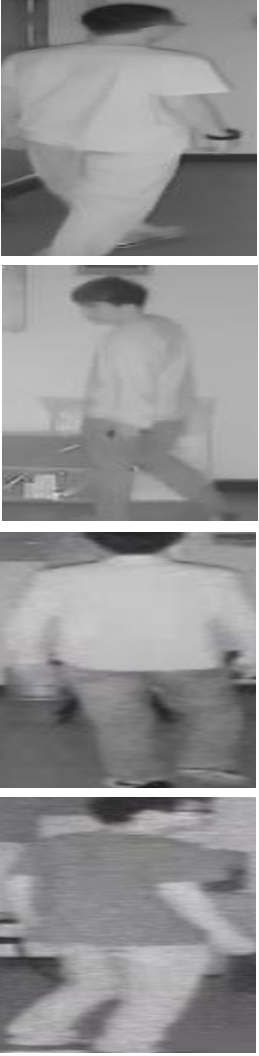}}  
 \subfloat[][]{\includegraphics[width=.15\linewidth,height=.9\linewidth]{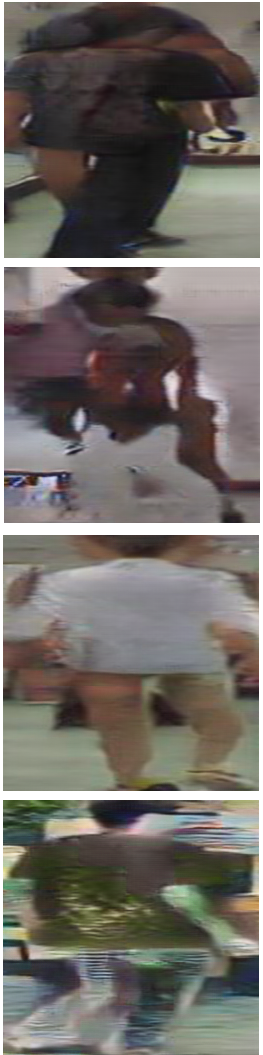}}   
  \caption{Examples of images generated with CycleGAN\cite{CycleGAN2017} for SYSU-MM01 dataset. Column (a) contains the real V, and column (b) is the translated visible image to infrared. Columns (c) and (d) are the real I images and translated versions of them to the V domain.   }  
  \label{fig:cycleGAN_result}
\end{figure} 

Additionally, generating the transformed images in testing makes the model trigger the GAN module and increases the model complexity. Also, such generated images may lose some information compared to the original images and degrade the matching performance. These images are shown in Fig. \ref{fig:cycleGAN_result}.

\subsection{Feature Visualization}
To demonstrate that the intermediate images in \AG training approach bridge between visible and infrared extracted features, we visualized the feature distributions of the intermediate, visible, and infrared representations during the training process. As shown in Fig. \ref{fig:umap_z}, the intermediate features are positioned between the visible and infrared features. The visible features gradually align with the intermediate feature space when we apply $\mathcal{L}{\text{cf}}$ and the first part of $\mathcal{L}{\text{dual}}$. As the model optimizes the second part of $\mathcal{L}_{\text{dual}}$, the infrared features progressively discard modality-specific information and converge toward the intermediate feature space. By the end of the training, the visible and infrared features had converged to a small region of the feature space, indicating that they had lost their modality-discriminative characteristics.

\begin{figure*}[!h]
  \centering
  \begin{minipage}{.45\linewidth}
    \centering
    \subfloat[Training epoch 0]{\includegraphics[width=\linewidth]{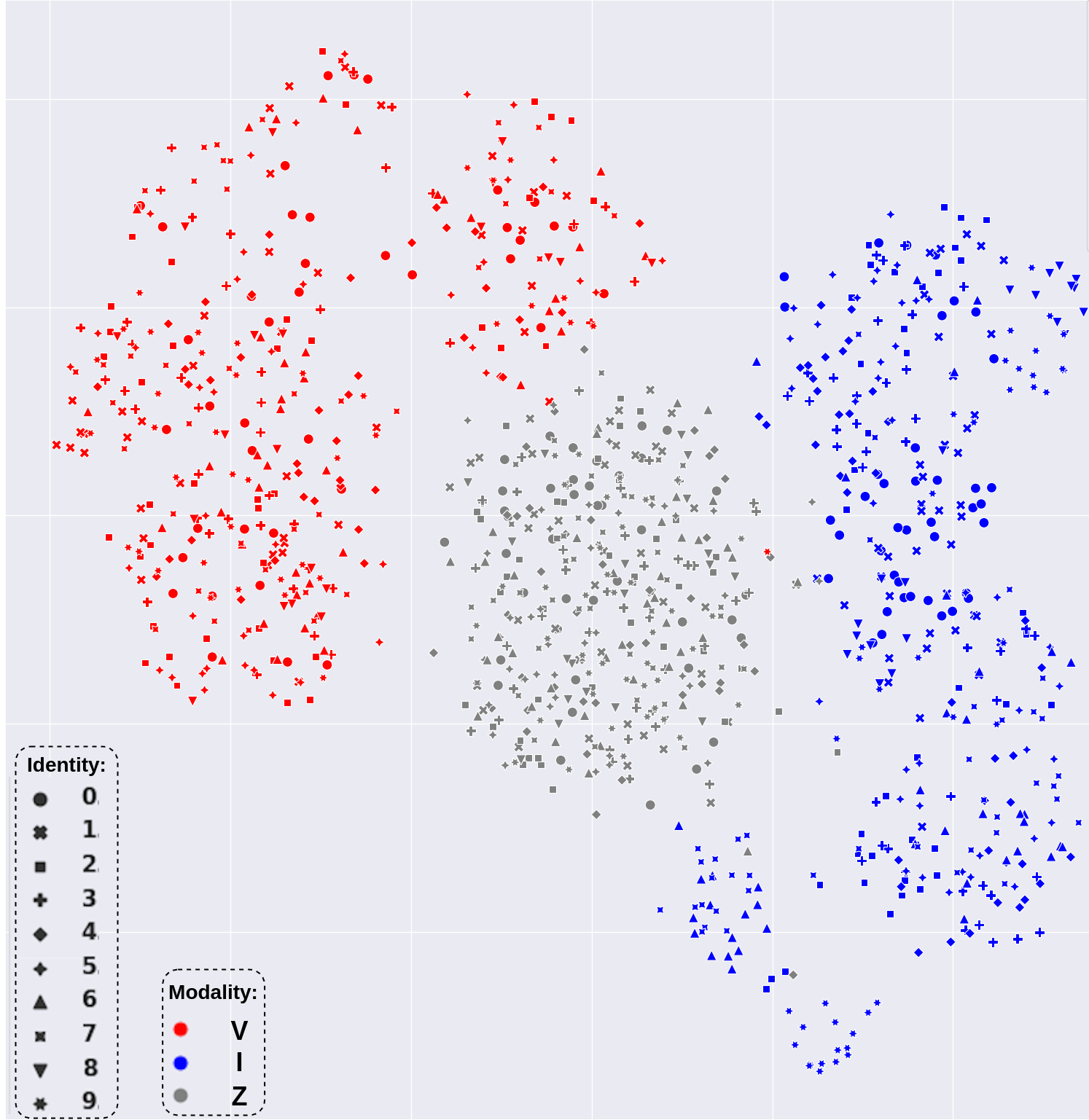}}
  \end{minipage}%
  \hfill
  \begin{minipage}{.45\linewidth}
    \centering
    \subfloat[Training epoch 30]{\includegraphics[width=\linewidth]{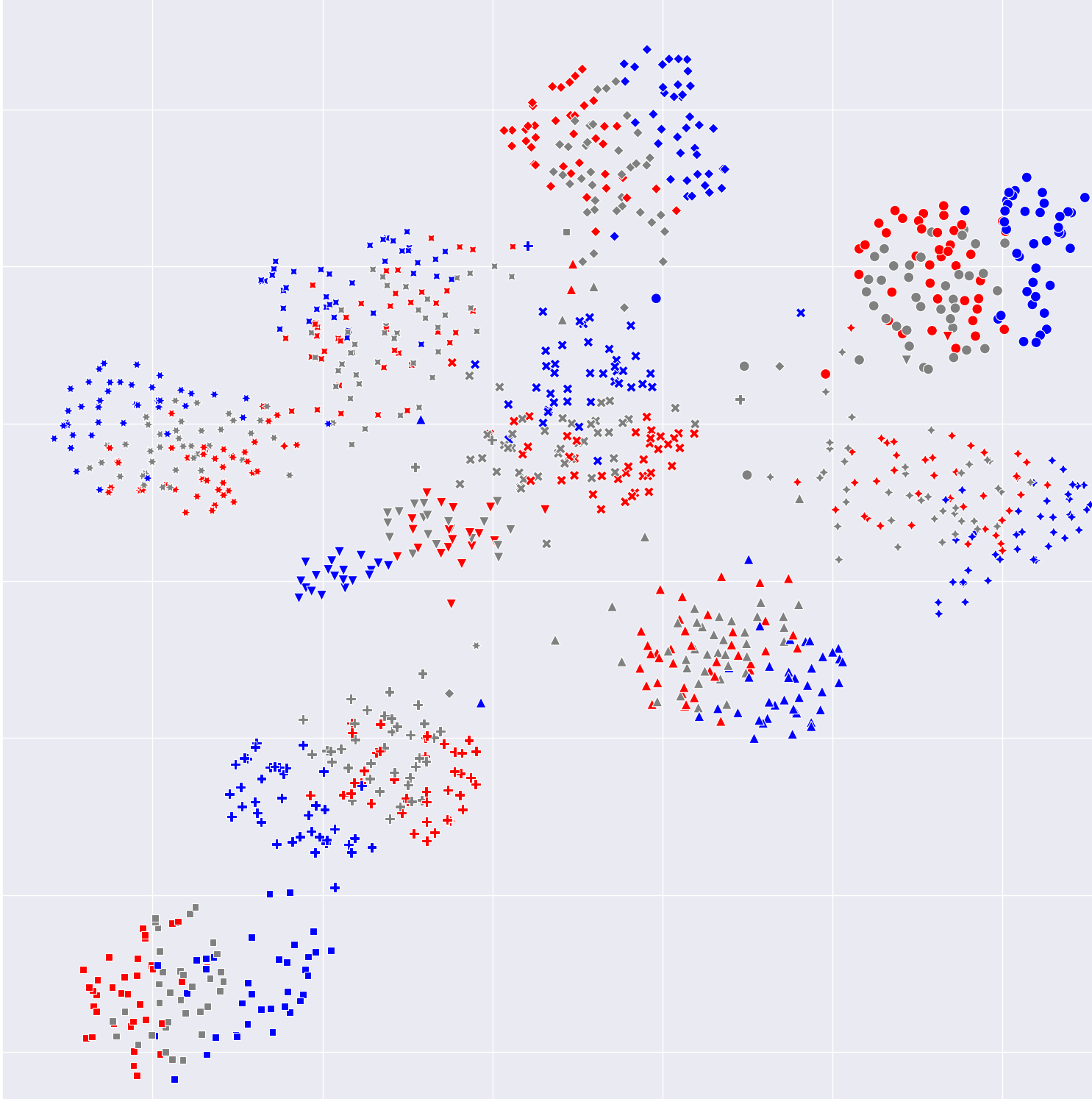}}
  \end{minipage}\par\medskip 
    \begin{minipage}{.45\linewidth}
    \centering
    \subfloat[Training epoch 50]{\includegraphics[width=\linewidth]{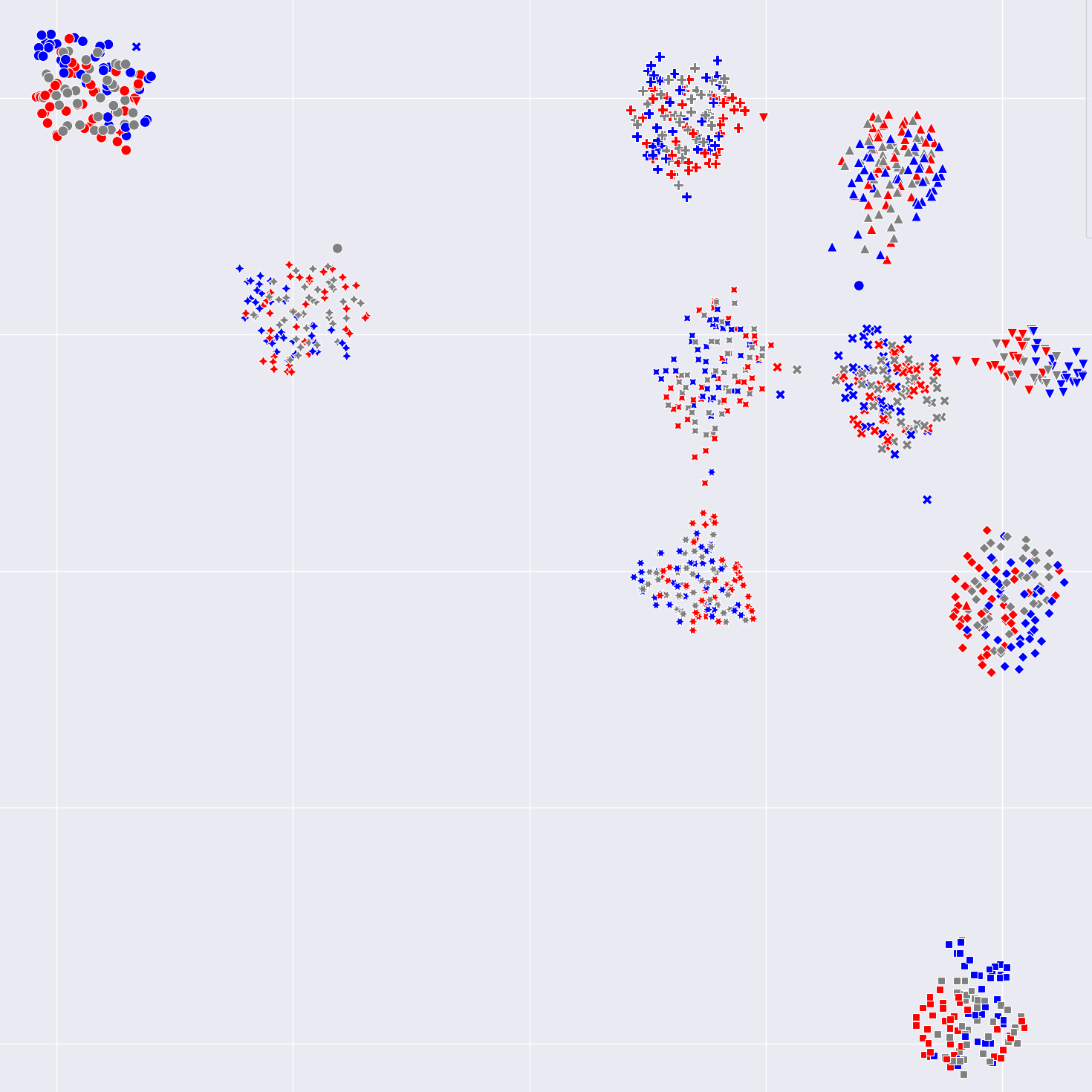}}
  \end{minipage}%
  \hfill
  \begin{minipage}{.45\linewidth}
    \centering
    \subfloat[Training epoch 70]{\includegraphics[width=\linewidth]{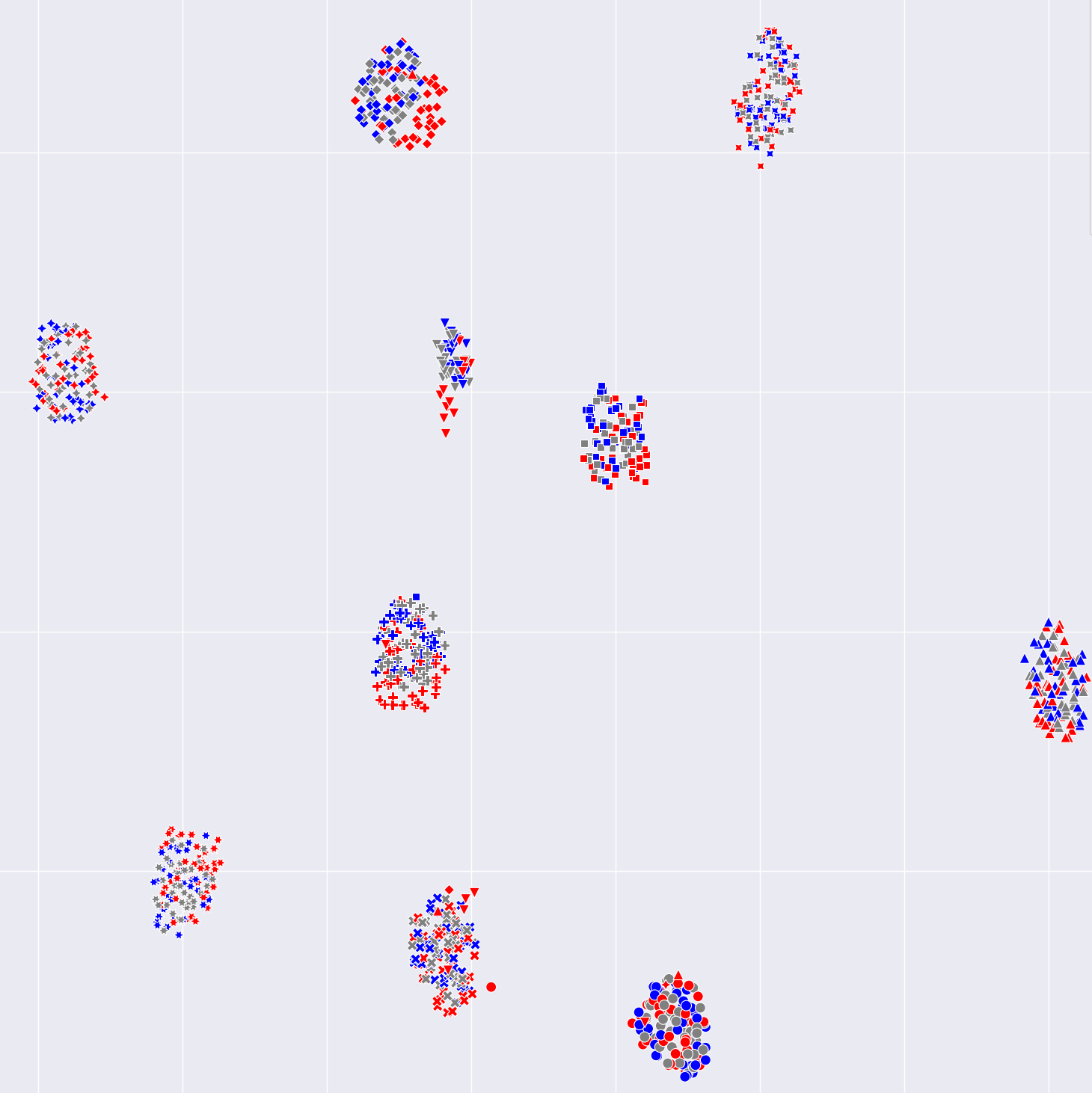}}
  \end{minipage}\par\medskip 
\caption{UMAP \cite{mcinnes2018umap} visual representation of ten randomly selected identities from the SYSU-MM01 dataset through training process. Identities are shown with different shapes, and modalities are shown with different colors (red and blue for visible and infrared and intermediate features are shown in gray color).}
\label{fig:umap_z}
\end{figure*}

\subsection{MMD Distance on extracted learned features:}
The objective of the intermediate feature is to serve as a bridge between visible and infrared information by providing privileged identity-discriminative information while eliminating modality-specific style information. To demonstrate the reduction in the domain gap achieved through the use of intermediate features, we applied the kernel Maximum Mean Discrepancy (MMD) function to the extracted features from visible, infrared, and intermediate images.
To measure MMD between modalities, we compute the center feature vector for each person within each modality and apply the MMD kernel to these vectors. Since, during the inference phase, infrared images are used as queries to the model in order to identify the matched person from the gallery, the distance between infrared (I) and intermediate (Z) features is crucial. As shown in Table \ref{tab:MMD2}, the use of intermediate features significantly reduces this distance compared to the baseline.

\begin{table}[!tb]
\small
\centering
\resizebox{\textwidth}{!}{
\begin{tabular}{|l||c|c|c||c|c|c|}
\hline 
\multirow{2}{*}{\textbf{Modality}}  & \multicolumn{3}{c||}{\textbf{Train}} & \multicolumn{3}{c|}{\textbf{Test}} \\ \cline{2-7}
                                    & \textbf{I-Z}          & \textbf{V-Z} &
                                    \textbf{I-V} ($\downarrow$ )     &\textbf{I-Z}         & \textbf{V-Z}   & \textbf{I-V} ($\downarrow$ )      \\ \hline  \hline
baseline          & -      & -  &  0.12  & -      & - & 0.53 \\ 
Grayscale (HAT \cite{HAT})          & 0.07      & 0.03  & 0.08   & 0.46      & 0.09 & 0.45 \\ 
Random(RPIG \cite{randLUPI}) & 0.05      & 0.04  & 0.06   & 0.41      & 0.11 & 0.38 \\ \hline
\AG (ours)                 &          0.03      & 0.05 & \textbf{0.05}    & 0.38     & 0.16 & \textbf{0.33} \\ \hline
\end{tabular}
\caption{MMD between V, I, and different intermediate learned features for 100 and 50 different persons from train and test sets, respectively, on the SYSU-MM01. Z images are the intermediate images reconstructed by $\mathcal{G}$.
}
\label{tab:MMD2}
}
\end{table}

\subsection{Comparing to Image-based Intermediate method in VI-ReID}
We also conducted a new experiment showing the qualitative and quantitative comparison between CycleGAN\cite{CycleGAN2017}, ThermalGAN\cite{kniaz2018thermalgan}, D$^2$RL\cite{D2RL}, HI-CMD\cite{HI-CMD} and MMN\cite{zhang2021towards} and our \AG.
We used the official code published by the authors of these methods and trained them with default hyperparameters.

\begin{table}[!ht]
\small
\centering

\resizebox{\textwidth}{!}{
\begin{tabular}{|l|c|c|c|c|c|c|}
\hline
\multicolumn{1}{|c|}{\multirow{2}{*}{\textbf{Method}}} & \multicolumn{2}{c|}{\textbf{Use}}                               & \multicolumn{2}{c|}{\textbf{Performance}}                           & \multicolumn{2}{c|}{\textbf{Complexity}} \\ \cline{2-7} 
\multicolumn{1}{|c|}{}                        & \textbf{Train} & \textbf{Test} & \textbf{R1(\%)} & \textbf{mAP(\%)} & \textbf{\#Para.} & Flops \\ \hline
CycleGAN\cite{CycleGAN2017}                &  \cmark    &  \xmark   & 42.05  & 41.17  & 24.8M          & 5.2G       \\
CycleGAN                &  \cmark &  \cmark & 26.91  & 28.75   & 42.5M            & 7.1G       \\ \hline 
thGAN\cite{kniaz2018thermalgan}                &  \cmark    &  \xmark   & -  & -  & -          & -       \\
thGAN                &  \cmark &  \cmark & 17.77  & 18.13   & 70.0M            & 13.7G       \\ \hline 
AlignGAN\cite{Wang_2019_ICCV_AlignGAN}                &  \cmark    &  \xmark   & 24.81  & 22.05  & 23.5M          & 5.2G       \\
AlignGAN          &  \cmark &  \cmark & 42.4  & 40.7   & 50.4M            & 9.6G       \\ \hline 
HiCMD\cite{HI-CMD}                &  \cmark &  \cmark & 34.94  & 35.94   & 40.1M            & 7.3G       \\ 
HiCMD                &  \cmark &  \cmark & 15.57  & 17.38   & 100.6M            & 13.5G       \\ \hline

MMN\cite{zhang2021towards}                &  \cmark    &  \xmark   & 65.14  & 61.57 & 72.9M          & 6.1G       \\
MMN(Only Z)                &  \cmark &  \cmark & 62.96  & 58.34   & 35.4M            & 6.2G   \\    
MMN (Z,I,V)               &  \cmark &  \cmark & 70.1  & 66.2   & 72.9M            & 12.1G   \\    \hline
\hline
\AG          &   \cmark    &  \xmark    & 72.23  & 70.58  & 24.8M           & 5.2G       \\
\AG                     &   \cmark    &   \cmark   & 65.76  & 62.23   & 52.8M            & 8.3G       \\ \hline

\end{tabular}
}
\captionsetup{width=\linewidth}
\caption{Impact of using generated images as intermediate versus using as visible in training or testing for SYSU dataset. \AG method uses them as intermediate in $\mathcal{L}_\textit{dual}$ and CycleGAN\cite{CycleGAN2017} uses as visible modality and applies the $\mathcal{L}_\textit{tri}$.} 
\label{table:GANS}
\end{table}

Table \ref{table:GANS} presents a comparison of open-sourced generative image-based methods for Visible-Infrared Person Re-Identification (VI-ReID) on the SYSU-MM01 dataset. For each method, we report the performance in terms of Rank-1 accuracy and mAP, alongside the computational complexity measured by the number of parameters and FLOPs. The table highlights the differences between using generated images as an intermediate representation during training versus as a matching modality during testing.  

The MMN \cite{zhang2021towards} approach utilizes generated intermediate images during both training and testing. Specifically, MMN extracts features from visible, infrared, and generated (Z) images, combining them for improved matching. While this strategy boosts performance, as evidenced by the 70.1\% Rank-1 accuracy and 66.2\% mAP in the (Z, I, V) setup, it significantly increases computational overhead, with a FLOPs count of 12.1G—almost double that of single-modality models. 

In contrast, our proposed \AG{} method treats generated images as privileged information during training, incorporating them into the dual loss \( \mathcal{L}_\textit{dual} \), without altering the backbone for testing. This leads to significantly lower computational overhead, with only 8.3G FLOPs in the test phase, while maintaining competitive performance. Notably, when the generated images (Z) are used as a query instead of infrared images (I), \AG{} achieves a Rank-1 accuracy of 65.76\%, outperforming MMN (62.96\%), which demonstrates the informativeness and bridging capability of the generated images. 

Furthermore, \AG{} not only delivers superior matching performance but also minimizes computational costs compared to methods like CycleGAN \cite{CycleGAN2017} and thermalGAN \cite{kniaz2018thermalgan}, which require substantial resources due to their use of generated images during both training and testing. Overall, \AG{} balances effectiveness and efficiency, making it a compelling solution for VI-ReID tasks with constrained computational budgets.

Also, we generated some samples from the SYSU-MM01 dataset in \autoref{fig:inter_result} for each method. Our \AG produces more realistic images with more details compared to others.

\begin{figure*}[!ht]
  \centering
 \subfloat[][]{
  \centering
  \frame{\includegraphics[width=1.75cm,height=13cm]{images/all-vh.jpg}}
  }
\subfloat[][]{
  \centering
  \frame{\includegraphics[width=1.75cm,height=13cm]{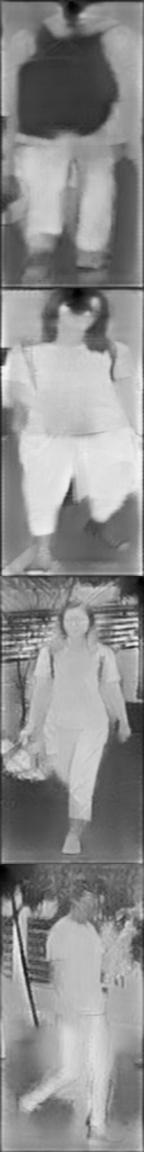}}
}
\subfloat[][]{
  \centering
  \frame{\includegraphics[width=1.75cm,height=13cm]{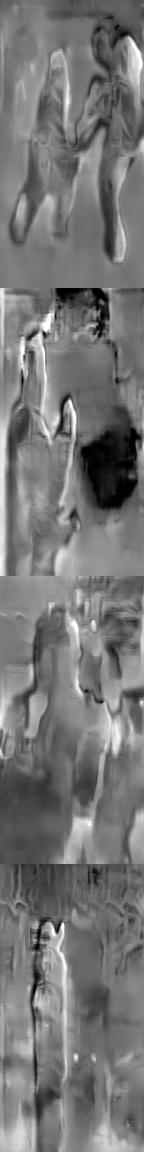}}
}%
\subfloat[][]{
  \centering
  \frame{\includegraphics[width=1.75cm,height=13cm]{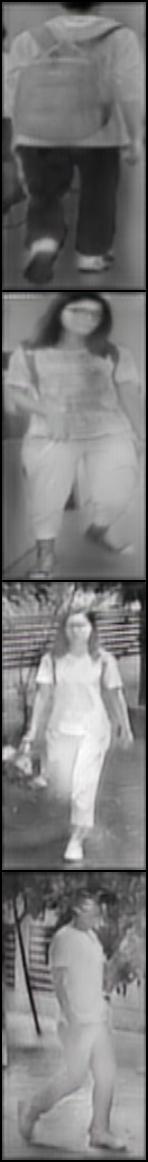}}
}%
\subfloat[][]{
  \centering
  \frame{\includegraphics[width=1.75cm,height=13cm]{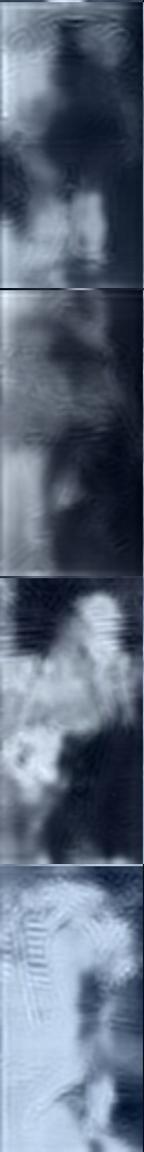}}
}%
\subfloat[][]{
  \centering
  \frame{\includegraphics[width=1.75cm,height=13cm]{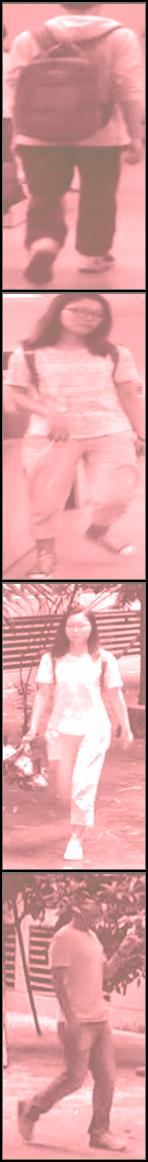}}
}%
\subfloat[][]{
  \centering
  \frame{\includegraphics[width=1.75cm,height=13cm]{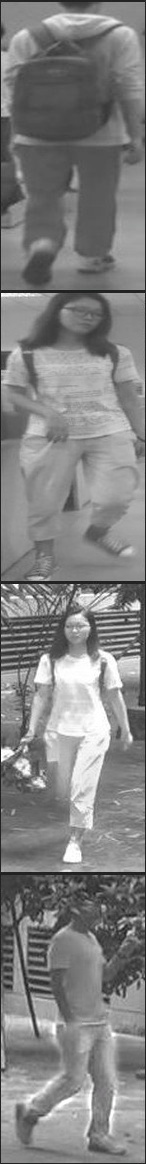}}
}%
  \caption{Examples of 4 intermediate images generated from (a) Visible by (b)CycleGAN\cite{CycleGAN2017},(c) ThermalGAN \cite{kniaz2018thermalgan}, (d) AlignGAN \cite{Wang_2019_ICCV_AlignGAN}, (e) HiCMD \cite{HI-CMD}, (f) MMN\cite{zhang2021towards} and (g) our \AG  on the SYSU-MM01 dataset. }
  \label{fig:inter_result}
\end{figure*}

\subsection{Incorporating to Extra SOTA Models}
Moreover, to compare with the requested method (PartMix \cite{PartMix23} and DEEN\cite{LLCM}), we apply our \AG approach to them and report the performance and efficacy. 
The PartMIX method \cite{PartMix23} utilizes part-based information by introducing an additional sub-module for part detection. To ensure a fair comparison, we integrated our \AG{} training approach into the SAAI method \cite{SAAI}, a publicly available part-based VI-ReID method, and retrained it from scratch. Similarly, \AG{} was added to the DEEN method \cite{LLCM} to evaluate its effectiveness on a non-part-based baseline. The results of these integrations on the SYSU-MM01 and RegDB datasets are shown in Table \ref{tab:part-Deen}.

\begin{table*}[!h]
\centering

\resizebox{0.75\linewidth}{!}{%
\begin{tblr}{
  colspec={|l|cc|cc||cc|cc||c|c|},
  cell{1}{1} = {r=3}{c},
  cell{1}{2} = {c=4}{c},
  cell{1}{6} = {c=4}{c},
  cell{1}{10} = {c=2}{c},
  cell{2}{2} = {c=2}{c},
  cell{2}{4} = {c=2}{c},
  cell{2}{6} = {c=2}{c},
  cell{2}{8} = {c=2}{c},
  cell{2}{10} = {r=2}{c},
  cell{2}{11} = {r=2}{c},
  hline{1,4,10} = {-}{},
  hline{2} = {2-11}{},
  hline{3} = {2-9}{},
}
\textbf{Method} & \textbf{SYSU-MM01} &  &  &  & \textbf{RegDB} &  &  &  & \textbf{Complexity} & \\
 & All Search &  & Indoor Search &  & V $\rightarrow$  I &  & I $\rightarrow$ V &  & \#Params  & FLOPs \\
 & \textbf{R1} & \textbf{mAP} & \textbf{R1} & \textbf{mAP} & \textbf{R1} & \textbf{mAP} & \textbf{R1} & \textbf{mAP} &  & \\
PartMix & 77.78 & 74.62 & 81.52 & 84.34 & 84.93 & 82.52 & 85.66 & 82.27 & 52.5M & 7.4G\\
\AG (ours) & 72.23 & 70.58 & 83.45 & 84.25 & 89.03 & 83.89 & 87.91 & 83.04 & 24.8M & 5.2G\\
\AG + Part & \textbf{77.35} & \textbf{78.56} & \textbf{84.60} & \textbf{88.91} & \textbf{91.22} & \textbf{91.15} & \textbf{91.3} & \textbf{92.5} & 52.5M & 7.4G\\ \hline \hline
DEEN & 74.7 & 71.8 & 80.3 & 83.3 & 91.1 & 85.8 & 89.5 & 83.4 & 89.0M & 39.8G\\
\AG(ours) & 72.23 & 70.58 & \textbf{83.45} & 84.25 & 89.03 & 83.89 & 87.91 & 83.04 & 24.8M & 5.2G\\
DEEN + \AG & \textbf{75.36} & \textbf{72.7} & 83.19 & \textbf{84.91} & \textbf{92.5} & \textbf{86.41} & \textbf{90.38} & \textbf{84.7} & 89.0M & 39.8G
\end{tblr}
}
\caption{Accuracy of integrating the proposed \AG with different baselines on the SYSU-MM01 (single-shot setting) and RegDB datasets.}
\label{tab:part-Deen}
\end{table*}

For the part-based baseline, integrating \AG{} with part information leads to a substantial improvement in performance. Specifically, the combined \AG{} + Part approach achieves a notable increase in mAP for SYSU-MM01 (All Search) from 74.62\% to 78.56\% while maintaining competitive Rank-1 accuracy. On RegDB, the combined approach further improves cross-modal performance, with a significant boost in mAP for the I$\rightarrow$V setting (from 82.52\% to 91.15\%). These results highlight the complementary nature of \AG{} when combined with part-based strategies, effectively enhancing both identity discrimination and modality bridging.

For the DEEN baseline, \AG{} integration results in consistent improvements across all settings. For example, in SYSU-MM01 (All Search), the DEEN + \AG{} combination improves mAP from 71.8\% to 72.7\%, and on RegDB (V$\rightarrow$I), mAP increases from 85.8\% to 86.41\%. These gains demonstrate the adaptability of \AG{} in improving feature representations across different baseline architectures. Notably, \AG{} achieves these enhancements with significantly lower complexity compared to other methods, as evidenced by its reduced parameter count and FLOPs.

These improvements demonstrate that \AG acts as a powerful enhancement tool when integrated with other models, further boosting their accuracy. Moreover, this integration does not introduce additional computational overhead during inference, making it a highly efficient solution.

    
